%% file: ijcai24.tex
\documentclass{article}
\usepackage{ijcai24}

\usepackage{times}
\usepackage{soul}
\usepackage{url}
\usepackage[hidelinks]{hyperref}
\usepackage[utf8]{inputenc}
\usepackage[small]{caption}
\usepackage{graphicx}
\usepackage{amsmath}
\usepackage{amsthm}
\usepackage{booktabs}
\usepackage[switch]{lineno}

\usepackage{subfigure}
\usepackage{upref}
\usepackage{mathtools}
\usepackage{amssymb}
\usepackage{mathtools}
\usepackage[vlined,linesnumbered,ruled]{algorithm2e}
\let\oldnl\nl%
\newcommand{\nonl}{\renewcommand{\nl}{\let\nl\oldnl}}
\usepackage{graphics}
\usepackage{microtype} %
\usepackage{enumitem}
\usepackage{thmtools, thm-restate}

\declaretheorem{lemma}

\usepackage{xcolor}
\hypersetup{
    colorlinks,
    linkcolor={red!50!black},
    citecolor={blue!50!black},
    urlcolor={blue!80!black}
}

\theoremstyle{remark}
\newtheorem{remark}{Remark}

\theoremstyle{definition}

\newtheorem{assumption}{Assumption}

\usepackage [operators,sets,landau,complexity]{cryptocode}

\usepackage{bm}
\renewcommand{\vec}[1]{\bm{#1}}
\DeclareMathOperator*{\E}{\mathbb{E}}

\def\NN{{\mathbb N}}
\def\RR{{\mathbb R}}

\usepackage{bbm}
\usepackage{physics2} %
\usephysicsmodule{ab}
\usepackage{derivative}
\usepackage{blkarray}
\newcommand{\1}{\mathbbm{1}}
\newcommand{\inprod}[2]{\left\langle #1, #2 \right\rangle}
\newcommand\numberthis{\addtocounter{equation}{1}\tag{\theequation}}

\usepackage{xspace}
\newcommand{\fedconpe}{\texttt{FedConPE}\xspace}

\newcommand{\citet}{\cite}

\title{FedConPE: Efficient Federated Conversational Bandits with Heterogeneous Clients}

\author{}

\author{
Zhuohua Li\and
Maoli Liu\And
John C.S. Lui\\
\affiliations
The Chinese University of Hong Kong
\emails
\{zhli, mlliu, cslui\}@cse.cuhk.edu.hk
}

\begin{document}

\maketitle

\begin{abstract}
  Conversational recommender systems have emerged as a potent solution for efficiently eliciting user preferences.
  These systems interactively present queries associated with ``key terms'' to users and leverage user feedback to estimate user preferences more efficiently.
  Nonetheless, most existing algorithms adopt a centralized approach.
  In this paper, we introduce \fedconpe, a phase elimination-based federated conversational bandit algorithm, where \(M\) agents collaboratively solve a global contextual linear bandit problem with the help of a central server while ensuring secure data management.
  To effectively coordinate all the clients and aggregate their collected data, \fedconpe uses an adaptive approach to construct key terms that minimize uncertainty across all dimensions in the feature space.
  Furthermore, compared with existing federated linear bandit algorithms, \fedconpe offers improved computational and communication efficiency as well as enhanced privacy protections.
  Our theoretical analysis shows that \fedconpe is minimax near-optimal in terms of cumulative regret.
  We also establish upper bounds for communication costs and conversation frequency.
  Comprehensive evaluations demonstrate that \fedconpe outperforms existing conversational bandit algorithms while using fewer conversations.
\end{abstract}

\section{Introduction}
\label{sec:introduction}

  In the contemporary digital era, content providers face a growing demand to serve a large number of customers with diverse geographical locations and disparate preferences around the globe.
  To align content with varied user interests, recommender systems are typically designed to continuously learn and adapt to users' preferences by collecting data from users' feedback.
  For instance, recommender systems that advertise products or news can collect users' real-time click rates and employ online learning algorithms to refine their recommendations continually.
  Due to the distributed nature of data and users, large-scale recommender systems usually deploy multiple servers distributed around the world.
  As a result, \emph{federated learning} is leveraged to localize training on each server and foster collaboration among different servers to enhance learning efficiency while ensuring that sensitive user information is neither transmitted nor centrally stored.

  One common challenge that recommender systems encounter is the ``cold start'' problem, meaning that recommendations may not be accurate for new users with little historical data.
  A viable solution to this problem is called \emph{active learning}, or \emph{conversational recommendation}~\cite{christakopoulou-2016-towards-conversational,sun-2018-conversational-recommender,lei-2020-conversational-recommendation,gao-2021-advances-challenges}, where the recommender system can proactively query users some questions and obtain feedback, thereby quickly accumulating data and eliciting user preference.
  This strategy has been widely used in many real-world applications.
  For example, the well-known large language model, ChatGPT, occasionally offers two responses to users, asking them to select the one they prefer.
  To illustrate, consider the real-world example in Figure~\ref{fig:example}, where ChatGPT is tasked with implementing a specific algorithm in Python.
  ChatGPT provides two implementations with different coding styles: one is object-oriented, and the other is procedural-oriented.
  The user is able to choose which one she prefers by clicking one of the two responses.
  Such interactions help ChatGPT to learn about the user's preferences, enabling it to continually refine its learning process and generate better responses for other, similar users.
  Note that the feedback is gathered from globally distributed users, necessitating ChatGPT to aggregate the data during the learning process.
  \begin{figure}[htb]
    \centering
    \includegraphics[width=\linewidth]{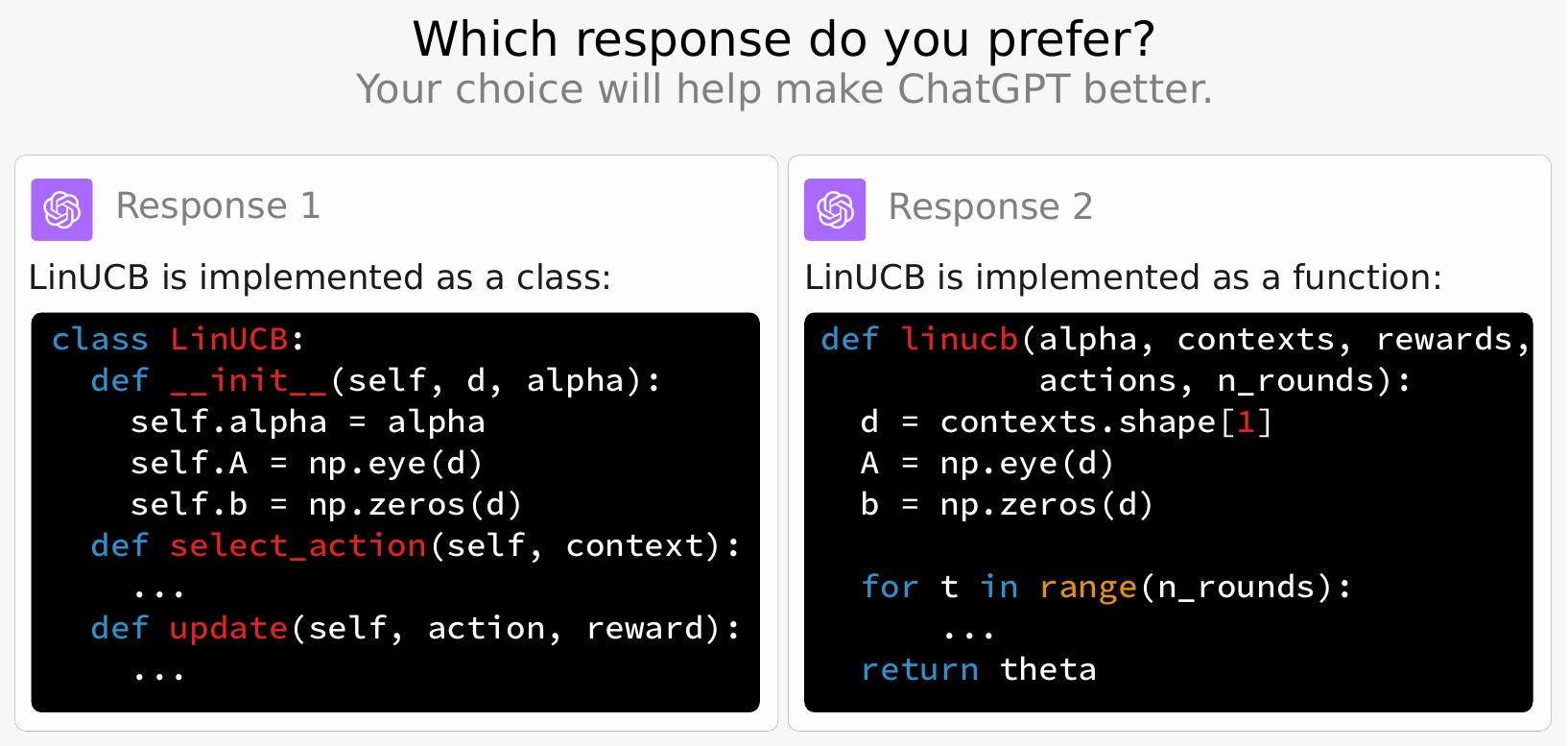}
    \caption{\label{fig:example} An example of conversational recommendation: ChatGPT sometimes gives two options for users to choose their preference.}
  \end{figure}

  In order to explore and understand user preferences, many recommender systems are equipped with \emph{contextual bandit} algorithms to balance the exploration and exploitation trade-off.
  A contextual bandit algorithm is a sequential decision-making algorithm, where in each round, it recommends items to the user and receives feedback/rewards (e.g., whether the user clicks on the recommended item or not).
  The objective of the algorithm is to devise an item recommendation strategy that maximizes the user's feedback over a long period.
  Recently, \emph{conversational contextual bandit} ~\cite{Zhang-Conversational-WWW20} has been proposed to model \emph{active learning} or \emph{conversational recommendation}.
  In this setting, besides receiving item-level feedback, the recommender occasionally prompts users with questions about \emph{key-terms} (e.g., movie genres, news topics) of recommended items and obtains key-term level feedback from users.
  Intuitively, the key-term level feedback reflects the user's preference for a category of items, allowing the recommender to speed up the learning process and eliminate the need to extensively explore all the items.

  While conversational recommendation is a robust solution for fast preference elicitation, the current formulation of conversational bandits still has the following problems.
  \begin{itemize}[leftmargin=*]
  \item Firstly, existing works on conversational bandits mostly follow the framework by \cite{Abbasi-Improved-2011}, which is designed to deal with linear bandit problems with \emph{infinitely many} arms (i.e., the recommended items).
  In practical recommender systems, however, the number of arms is typically finite.

  \item Secondly, existing studies about conversational bandits~\cite{Zhang-Conversational-WWW20,Wang-2023-Efficient} rely on a deterministic function \(b(t)\) to control the frequency of conversations.
  Moreover, they fix the sequence of item recommendation and conversations, i.e., the recommender can only initiate \(P\) conversations \emph{at once} per \(Q\) recommendations for some constant \(P\) and \(Q\).
  This does not align with the best practice of conversational recommendation in reality, imposes unnecessary constraints on the recommender, and potentially diminishes the user experience.

  \item Thirdly, existing literature about conversational bandits solely considers the single-client scenario, neglecting the crucial need for multi-client algorithms in real-world distributed systems.
  How to effectively coordinate data collected across different clients and appropriately initiate conversations remains an unsolved problem.
  \end{itemize}

  To address the challenges mentioned above, this paper considers the federated conversational bandit problem with \(M\) clients and \emph{finite} arm sets.
  We propose \fedconpe (\underline{Fed}erated \underline{Con}versational \underline{P}hase \underline{E}limination algorithm), where we follow the classical phase elimination algorithm~\cite{lattimore-2020-bandit-algorithms} to handle the finite arm setting and improve the regret upper bound to \(\mathcal{\widetilde{O}}(\sqrt{dMT})\).
  Instead of deterministically pre-defining the conversation frequency, our algorithm adaptively determines whether a conversation is needed according to the collected data.
  Moreover, it does not mandate the order of conversations and recommendations.
  Although there are existing works~\cite{huang-2021-federated,lin-2023-federated} proposed for federated linear bandit with finite arms, we highlight that integrating the conversational setting is nontrivial and has several advantages.
  On one hand, the existing algorithm requires all the clients to upload their active arm sets to the central server.
  This not only leaks sensitive information, such as clients' available choices, to the central server, but also increases the communication costs.
  On the other hand, the existing algorithm relies on computing the \emph{multi-client G-optimal design}~\cite{huang-2021-federated} by the server, which is known to be computationally prohibitive; thus it is challenging to achieve real-time recommendation.
  In contrast, for \fedconpe, with the help of conversations, each client only needs to locally compute its own \emph{G-optimal design} on its active arm set, which can be efficiently solved by the well-known Frank-Wolfe algorithm under suitable initialization~\cite{frank-wolfe-1956,todd-2016-minimum-volume}.
  Thus, both computation and communication costs are reduced.
  We refer interested readers to Appendix~\ref{sec:related-work} for a more comprehensive overview of related work.
  To the best of our knowledge, among all the conversational bandit literature, our algorithm is the first to consider the federated setting and the first to give an analysis on lower bounds.

  In summary, our contributions are listed as follows.
  \begin{itemize}[leftmargin=*]
  \item We propose \fedconpe, an algorithm for federated conversational bandits with finite arm sets, which can facilitate collaboration among all the clients to improve the efficiency of recommendations and conversations.
  \item We exploit the property of the conversational setting to enhance the efficiency in both computation and communication compared with existing work in federated linear contextual bandits with finite arm sets.
  \item We show that \fedconpe is nearly minimax optimal by theoretically proving the regret upper bound \(\mathcal{\widetilde{O}}(\sqrt{dMT})\) and the matching lower bound \(\Omega(\sqrt{dMT})\) for conversational bandits. We also prove the communication cost upper bound of \(\mathcal{O}(d^2M \log T)\), and show that conversations only happen for a small number of rounds.
  \item We conduct extensive experiments on both synthetic and real-world datasets to demonstrate that our algorithm achieves lower cumulative regret and uses fewer conversations than prior works.
  \end{itemize}

\section{Problem Formulation}
\label{sec:models}
\subsection{Federated Conversational Bandits}
  We define \([M] := \set{1, \dots, M}\) for \(M \in \NN^{+}\).
  For any real vector \(\vec{x}, \vec{y}\) and positive semi-definite (PSD) matrix \(\vec{V}\), \(\|\vec{x}\|\) denotes the \(\ell_2\) norm of \(\vec{x}\), \(\inprod{\vec{x}}{\vec{y}}=\vec{x}^\mathsf{T} \vec{y}\) denotes the dot product of vectors, and \(\|\vec{x}\|_{\vec{V}}\) denotes the Mahalanobis norm \(\sqrt{\vec{x}^\mathsf{T} \vec{V} \vec{x}}\).
  We consider the federated conversational multi-armed bandits setting, where there are \(M\) clients and a central server.
  Each client \(i\) has a finite set of arms (i.e., recommended items) \(\mathcal{A}_i \subset \RR^d\) and \(|\mathcal{A}_i|=K\).
  Note that clients are heterogeneous, meaning that different clients may have different arm sets, and we assume that arms in \(\mathcal{A}_i\) span \(\RR^d\).
  At each time step \(t \in [T]\), a client \(i \in [M]\) selects an arm \(\vec{a}_{i,t} \in \mathcal{A}_i\) and receives a reward \(x_{i,t}\) from the environment.
  The reward is assumed to have a linear structure: \(x_{i,t} = \vec{a}_{i,t}^\mathsf{T} \vec{\theta}^{*} + \eta_{i,t}\), where each arm \(\vec{a}_{i,t} \in \RR^d\) is represented as a feature vector, \(\vec{\theta}^{*} \in \RR^d\) is an unknown arm-level preference vector that all the clients are trying to learn, and \(\eta_{i,t}\) is a noise term.
  Our objective is to design a learning policy to choose the arms for each time step \(t=1, \dots, T\) such that the cumulative regret, which is the difference of cumulative rewards between our policy and the best policy across all clients, is as small as possible.
  The cumulative regret is defined as:
  \begin{equation}\label{eq:regret}
    R_M(T) = \sum_{i=1}^{M} \sum_{t=1}^{T} \left( \max_{\vec{a} \in \mathcal{A}_i} \vec{a}^\mathsf{T} \vec{\theta}^{*} - \vec{a}_{i,t}^\mathsf{T} \vec{\theta}^{*} \right).
  \end{equation}

  Besides obtaining feedback by pulling arms, the central server can also occasionally query the users of each client and obtain feedback on some conversational \emph{key terms} to accelerate the elicitation of user preferences.
  In particular, a ``key term'' refers to a keyword or topic related to a subset of arms.
  For example, the key term ``programming language'' may be related to arms such as ``C/C++'', ``Python'', ``Java'', etc.
  Let \(\mathcal{K} \subset \RR^d\) denote the finite set of key terms, with each element \(\vec{k} \in \mathcal{K}\) being a fixed and known feature vector.
  The same as previous works~\cite{Zhang-Conversational-WWW20,Wang-2023-Efficient}, the server can initiate a conversation with client \(i\) by presenting a key term \(\vec{k}_{i,t} \in \mathcal{K}\), and the client's feedback is modeled as \(\widetilde{x}_{i,t} = \vec{k}_{i,t}^\mathsf{T} \widetilde{\vec{\theta}}^{*}+\widetilde{\eta}_{i,t}\), where \(\widetilde{\vec{\theta}}^{*}\) is an unknown key-term level preference vector and \(\widetilde{\eta}_{i,t}\) is a noise term.
  It is important to note that in conversational bandit literature, the arm and key-term level preference vectors \(\vec{\theta}^{*}\) and \(\widetilde{\vec{\theta}}^{*}\) are either assumed to be very close or exactly the same.
  Here, we follow the formulation of \citet{Wang-2023-Efficient} and assume that the two preference vectors are identical.

  One significant difference between our formulation and the prior conversational bandit literature is that, previous works depend on a deterministic function \(b(t)\) to govern the frequency of conversations, typically defined as a linear or logarithmic function of \(t\).
  This means that the recommender only periodically launches conversations without considering whether the user preference has been thoroughly estimated, potentially impacting the user experience in practice.
  In contrast, as illustrated in Section~\ref{sec:algorithms}, our algorithm adaptively determines whether to initiate conversations based on the accumulated data.
  If the estimation of the unknown preference vector \(\vec{\theta}^{*}\) is already sufficiently accurate, the recommender will not prompt conversations to users.

  In the following, we list and explain our assumptions.
\begin{assumption}\label{assumption:1}
  We assume the feature vectors for both arms and key terms are normalized, i.e., \(\|\vec{a}\|=\|\vec{k}\|=1\) for all \(\vec{a} \in \mathcal{A}_i\) and \(\vec{k} \in \mathcal{K}\).
  We also assume the unknown preference vector is bounded, i.e., \(\|\vec{\theta}^{*}\| \leq 1\), and \(\eta_{i,t}, \widetilde{\eta}_{i,t}\) are both 1-subgaussian random variables.
\end{assumption}

\begin{assumption}\label{assumption:2}
  We assume that the elements in the key term set \(\mathcal{K}\) are sufficiently rich, such that for any unit vector \(\vec{v} \in \RR^d\), there exists a key term \(\vec{k} \in \mathcal{K}\) satisfying \(\vec{k}^\mathsf{T}\vec{v} \geq C\), where \(C \in (0,1]\) is a universal constant that is close to 1.
\end{assumption}

Assumption~\ref{assumption:1} aligns with the conventional assumption made in most of the linear bandits literature~\cite{Abbasi-Improved-2011}.
Assumption~\ref{assumption:2}, serving as a technical assumption, ensures that the feature vectors of key terms are sufficiently distributed across the entire feature space, such that they are representative enough to help each client minimize uncertainty during the learning process.
This assumption is mild and readily attainable.
For example, it is satisfied when the key term set \(\mathcal{K}\) contains an orthonormal basis in \(\RR^d\).

\subsection{Communication Model}
  We adopt a star-shaped communication paradigm, i.e., each client can communicate with the server by uploading and downloading data with zero latency, but clients cannot directly communicate with each other.
  In addition, we assume that clients and the server are fully synchronized, meaning that at each time \(t\), each client pulls and only pulls one arm.
  Considering that the regret definition does not take the key-term level feedback into account, without loss of generality, we assume that querying key terms does not increment the time step \(t\), thereby allowing querying a key term and pulling an arm to happen simultaneously.
  To prevent ambiguity of the notation \(\vec{k}_{i,t}\), we assume that at most one key term is queried per time step \(t\) for each client.
  Consistent with existing works~\cite{huang-2021-federated}, we define the communication cost as the number of scalars (either integers or real numbers) transmitted between the server and clients.

\section{Algorithms \& Theoretical Analysis}
\label{sec:algorithms}

\subsection{Challenges}
\label{sec:challenges}
  Our work is based on the classical single-client phase elimination algorithms for linear bandits (Section 22 in \citet{lattimore-2020-bandit-algorithms}), which we refer to as \texttt{LinPE}.
  In this algorithm, the learner estimates the unknown preference vector \(\vec{\theta}^{*}\) using the \emph{optimal design} of least squares estimators.
  Specifically, the learner minimizes the maximum prediction variance by computing the \emph{G-optimal design} \(\pi: \mathcal{A} \to [0,1]\), which is a distribution over the arm set \(\mathcal{A}\).
  As shown in Lemma~\ref{lemma:kiefer-wolfowitz}, the G-optimal design \(\pi\) guarantees that the prediction variance \(g(\pi) \leq d\).
  Then the learner plays arms according to the distribution \(\pi\), estimates the unknown parameter \(\vec{\theta}^{*}\), and eliminates inferior arms based on the estimation.
  As a result, the regret of \texttt{LinPE} scales as \(\mathcal{\widetilde{O}}(\sqrt{dT})\).

  \begin{lemma}[\protect\citeauthor{kiefer-wolfowitz-1960}]\label{lemma:kiefer-wolfowitz}
    Assume that \(\mathcal{A} \subset \RR^d\) is compact and \(\text{span}(\mathcal{A}) = \RR^d\), let \(\pi: \mathcal{A} \to [0,1]\) be a distribution on \(\mathcal{A}\) and define \(\vec{V}(\pi)=\sum_{\vec{a} \in \mathcal{A}} \pi(\vec{a})\vec{a} \vec{a}^\mathsf{T}\), then there exists a minimizer \(\pi^{*}\) of \(g(\pi)=\max_{\vec{a} \in \mathcal{A}} \|\vec{a}\|_{\vec{V}(\pi)}\) such that \(g(\pi^{*}) = d\) and \(|\text{Supp}(\pi^{*})| \leq d(d+1)/2\).
  \end{lemma}

  To extend \texttt{LinPE} algorithm to the federated setting, one straightforward and intuitive approach is to let each client locally run \texttt{LinPE}, and the central server aggregates data collected from them to estimate \(\vec{\theta}^{*}\).
  Intuitively, with the server now possessing more data, the estimation should be more accurate.
  However, this simple idea \emph{will not work} because of client heterogeneity, namely, different clients may have different arms sets.
  Specifically, the G-optimal design computed by one client may not contain useful information for another client, and hence the estimation made by the aggregated data is likely to be biased, meaning that it may be accurate only for a specific client, but not others.
  As a result, even if the central server aggregates data from all the clients, it may not help improve the regret.
  One can verify that this intuitive algorithm only results in a regret bound of \(\mathcal{\widetilde{O}}(M \sqrt{dT})\), which can be trivially achieved by individually running the \texttt{LinPE} algorithm for each client without any communication.

\subsection{Intuition of Utilizing Key Terms}
  As we elaborated above, the main challenge of federated conversational bandits is that aggregating data according to each client's local G-optimal design does not necessarily lead to a better estimation of \(\vec{\theta}^{*}\) because of client heterogeneity.
  To mitigate this issue, the central server can leverage key terms to further explore the directions lacking information until the estimation is accurate across all directions in the feature space.

  Specifically, the matrix \(\vec{V}(\pi)\) in Lemma~\ref{lemma:kiefer-wolfowitz} (referred to as \emph{information matrix}) contains the ``information'' of the feature space.
  Each eigenvector of \(\vec{V}(\pi)\) represents a direction in the feature space, and its associated eigenvalue indicates how much information it encompasses.
  That is, an eigenvector associated with a larger eigenvalue points towards a direction where \(\vec{V}(\pi)\) contains more information, leading to a more precise estimation in this direction.
  Therefore, ensuring that all the eigenvalues of \(\vec{V}(\pi)\) are reasonably large guarantees that the model can make accurate predictions with high certainty throughout the entire feature space.

  Based on the above intuition, our algorithm enables the server to communicate with each client, identify directions with deficient information, and accordingly select appropriate key terms for each client.
  Then, the server can aggregate data from all clients and minimize the uncertainty across the entire space.
  All of the above intuitions are rigorously formulated in our theoretical analysis (see details in Appendix~\ref{sec:proof-regret}).

\subsection{\fedconpe Algorithms}
  The algorithms for the clients and the server are described in Algorithm~\ref{algo:client} and Algorithm~\ref{algo:server}, respectively.
  In the following, we define \(\mathcal{T}_{i,\vec{a}}^{\ell}\) as the set of time steps during which client \(i\) plays arm \(\vec{a}\) in phase \(\ell\), and define \(\mathcal{T}_{i}^{\ell} =\bigcup_{\vec{a} \in \mathcal{A}_i^{\ell}} \mathcal{T}_{i,\vec{a}}^{\ell}\).
  Similarly, \(\mathcal{\widetilde{T}}_{i,\vec{k}}^{\ell}\) represents the set of times steps during which client  \(i\) plays key term \(\vec{k}\) in phase \(\ell\), with \(\mathcal{\widetilde{T}}_{i}^{\ell} =\bigcup_{\vec{k} \in \mathcal{K}_i^{\ell}} \mathcal{\widetilde{T}}_{i,\vec{k}}^{\ell}\).
  We also denote \(\lambda_{\vec{v}}\) as the eigenvalue associated with eigenvector \(\vec{v}\) and define \(\varepsilon_{\ell} = 2^{-\ell}\).

\subsubsection{Client-side Algorithm}
\label{sec:client-algorithm}
  As shown in Algorithm~\ref{algo:client}, in each phase \(\ell\), each client \(i\) first computes its G-optimal design \(\pi_i^{\ell}\), a distribution among its available arms \(\mathcal{A}_i^{\ell}\) (called \emph{active} arms).
  We note that this can be effectively computed by the well-known Frank-Wolfe algorithm under an appropriate initialization.
  The main purpose of the G-optimal design is to minimize the maximum prediction variance so as to make the estimation of \(\vec{\theta}^{*}\) as precise as possible.
  At line~\ref{line:diagonalization}, by diagonalizing the information matrix \(\vec{V}_i^{\ell}(\pi_i^{\ell})\), the client can check all directions (represented by eigenvectors) to see if they lack adequate certainty (Line~\ref{line:check-eigenvalue}), then uploads the eigenvalue-eigenvector pairs to the central server.
  The server subsequently returns a set of key terms \(\mathcal{K}_i^{\ell}\) and the corresponding repetition times \(\set{n_{\vec{k}}}_{\vec{k} \in \mathcal{K}_i^{\ell}}\) (Line~\ref{line:download-key-terms}).
  Then client \(i\) plays arms and key terms for the required number of times, and shares its progress by uploading the local data to the server (Line~\ref{line:upload-data}) so that it can benefit other clients.
  Note that the order of playing each arm\slash key term is not important.
  I.e., within a phase, the client can freely arrange the order as long as ensuring each arm\slash key term is played the required times.
  This flexibility affords each client more freedom to change its recommendation strategy as needed.
  Finally, the client downloads the estimated preference parameter \(\widehat{\vec{\theta}}_{\ell}\), and updates its active arm set by eliminating inferior arms according to the estimation (Line~\ref{line:elimination}).
  We emphasize that the sensitive data for each client (e.g., which arm is played and the corresponding reward) is not transmitted to the central server as the client only shares locally aggregated data (\(\vec{G}_i^{\ell}\) and \(\vec{W}_i^{\ell}\)).

  \input{algorithms/client.tex}

\subsubsection{Server-side Algorithm}
\label{sec:server-algorithm}
  As shown in Algorithm~\ref{algo:server}, the server's role is to find appropriate key terms in \(\mathcal{K}\) that can minimize the uncertainty for each client \(i\), and aggregate data from the clients to estimate the preference parameter.
  Upon receiving the eigenvalues and eigenvectors from each client, the server searches all the key terms and finds the closest one in terms of the inner product (Line~\ref{line:find-key-term}).
  Then the selected key term \(\vec{k}\) and the corresponding repetition times \(n_{\vec{k}}\) are sent back to the client (Line~\ref{line:send-back-key-terms}).
  Finally, the server aggregates data from all the clients and estimates the unknown preference parameter via linear regression (Line~\ref{line:aggregate-data} and Line~\ref{line:estimate-theta}).
  Note that estimating \(\widehat{\vec{\theta}}_{\ell}\) requires computing matrix inverse.
  Without loss of generality, we assume that \(\vec{G}\) is always invertible when the associated feature vectors span \(\RR^d\).
  When they do not span \(\RR^d\), we can always reduce the number of dimensions.

\input{algorithms/server.tex}

\subsection{Theoretical Analysis}
\label{sec:theoretical-analysis}
This section presents the theoretical results of \fedconpe, including its cumulative regret, communication costs, and the number of conversations.
The proofs of Theorem~\ref{thm:regret}, \ref{thm:lowerbound}, \ref{thm:communication}, and \ref{thm:conversation} are given in Appendix~\ref{sec:proof-regret}, \ref{sec:proof-lowerbound}, \ref{sec:proof-communication}, and \ref{sec:proof-conversation}, respectively.
\begin{restatable}[Regret upper bound]{theorem}{restateregret}\label{thm:regret}
  With probability at least \(1-\delta\), the cumulative regret scales in \(\mathcal{O}\ab(\sqrt{dMT\log \frac{KM\log T}{\delta}})\).
\end{restatable}

\begin{remark}\label{remark:regret}
  When there is only one client (i.e., \(M=1\)), our setting reduces to the non-federated conversational bandits and the regret scales in \(\mathcal{\widetilde{O}}(\sqrt{dT})\).
  It improves the result of prior works~\cite{Zhang-Conversational-WWW20,Wang-2023-Efficient}, which is \(\mathcal{\widetilde{O}}(d\sqrt{T})\).
  The improvement stems from utilizing phase elimination on finite arm sets.
  Also, our regret bound coincides with that presented by \citet{huang-2021-federated}.
  Although \fedconpe relies on the conversational setting, it sidesteps the computationally prohibitive \emph{multi-client G-optimal design}, making it easier to achieve real-time recommendations.
\end{remark}

\begin{restatable}[Regret lower bound]{theorem}{restatelowerbound}\label{thm:lowerbound}
  For any policy that chooses at most one key term at each time step, there exists an instance of the federated conversational bandits such that the policy suffers an expected regret of at least \(\Omega(\sqrt{dMT})\).
\end{restatable}

\begin{remark}
  Theorem~\ref{thm:regret} and Theorem~\ref{thm:lowerbound} imply that \fedconpe is minimax optimal up to a logarithmic factor.
\end{remark}

\begin{restatable}[Communication cost]{theorem}{restatecommunication}\label{thm:communication}
  The communication cost scales in \(\mathcal{O}(d^2M\log T)\).
\end{restatable}

\begin{remark}\label{remark:communication}
  Compared with \citet{huang-2021-federated}, whose communication cost is \(O(d^2MK\log T)\), our result demonstrates an improvement by a factor of \(K\).
  This is because \fedconpe does not require each client to upload its active arm set (whose cardinality is of size \(\mathcal{O}(K)\)).
  Instead, each client independently collects data based on its own arms, and only communicates the aggregated data to the server.
  This not only reduces communication overhead but also enhances privacy preservation by limiting the data shared with the server.
\end{remark}

\begin{restatable}[Conversation frequency upper bound]{theorem}{restateconversation}\label{thm:conversation}
  For any client \(i \in [M]\) and phase \(\ell \in [L]\), let \(\beta = \lambda_{\text{min}}\ab(V_i^{\ell}(\pi_i^{\ell}))\).
  We have: (a) If \(\beta\geq \frac{3}{4(1-\varepsilon_{\ell}^2)dN}\), no conversations will be initiated. (b) Otherwise, the fraction of conversations is at most \(\frac{\frac{3}{4(1-\varepsilon_{\ell}^2)}-dN\beta}{NC^2}\) relative to the total phase length.
\end{restatable}
\begin{remark}\label{remark:conversation}
  Even in the worst case, where \(\beta=0\), the proportion of conversations is at most \(\mathcal{O}(1/(NC^2))\).
  This theorem indicates that \fedconpe initiates only a small number of conversations over the time horizon.
\end{remark}

\section{Performance Evaluation}
\label{sec:evaluation}
  In this section, we conduct extensive experiments to demonstrate the effectiveness of our algorithm.
  Specifically, we aim to answer the following research questions:
  \begin{enumerate}[leftmargin=*]
  \item For both the single-client and multi-client settings, does our algorithm \fedconpe outperform existing state-of-the-art algorithms for conversational contextual bandits?
  \item How do the number of clients \(M\) and the arm set size \(K\) affect the performance of \fedconpe?
  \item Does our algorithm use fewer conversations in practice?
  \end{enumerate}

\subsection{Experimental Settings}
\label{sec:experimental-settings}

\subsubsection{Datasets}
  In light of previous studies, we generate a synthetic dataset and use the following three real-world datasets.
  The details of data generation and preprocessing are postponed to Appendix~\ref{sec:data-generation-preprocessing} in the extended version of this paper~\cite{li-2024-fedconpe}.
  \begin{itemize}[leftmargin=*]
    \item \textbf{MovieLens-25M~\cite{harper-2015-movielens}}: A dataset from MovieLens, a movie recommendation website. It contains 25,000,095 ratings across 62,423 movies created by 162,541 users.
    \item \textbf{Last.fm~\cite{cantador-2011-second-workshop}}: A dataset from an online music platform Last.fm. It contains 186,479 tag assignments, interlinking 1,892 users with 17,632 artists.
    \item \textbf{Yelp}\footnote{\url{https://www.yelp.com/dataset}}: A dataset from Yelp, a website where users post reviews for various businesses.
    It contains 6,990,280 reviews for 150,346 businesses created by 1,987,897 users.
  \end{itemize}

  \subsubsection{Baseline Algorithms}
  We select the following \ algorithms from existing studies as the baselines that we will compare with.
  \begin{itemize}[leftmargin=*]
  \item \texttt{LinUCB}~\cite{li-2010-a-contextual,Abbasi-Improved-2011}: The standard linear contextual bandit framework designed for \emph{infinite} arm sets.
    It does not consider the conversational setting so it only has arm-level feedback.
  \item \texttt{ConUCB}~\cite{Zhang-Conversational-WWW20}: The original algorithm proposed for conversational contextual bandits. It queries key terms when conversations are allowed and leverages conversational feedback on key terms to improve learning speed.
  \item \texttt{Arm-Con}~\cite{christakopoulou-2016-towards-conversational}: The algorithm that initiates conversations based on arms instead of key terms. The arm selection follows \texttt{LinUCB}.
  \item \texttt{ConLinUCB}~\cite{Wang-2023-Efficient}: A series of algorithms that change the key term selection strategy.
    It contains 3 algorithms: \texttt{ConLinUCB-BS} computes the \emph{barycentric spanner} of key terms as an exploration basis.
    \texttt{ConLinUCB-MCR} selects key terms that have the largest confidence radius.
    \texttt{ConLinUCB-UCB} uses a \texttt{LinUCB}-like strategy to choose key terms that have the largest upper confidence bounds.
  \end{itemize}

\subsection{Evaluation Results}
\label{sec:evaluation-results}

\subsubsection{Cumulative Regret for a Single Client}
  First of all, we evaluate the single-client scenario and compare \fedconpe with all the six baseline algorithms in terms of cumulative regret.
  We randomly select 10 users and calculate their cumulative regret over $T=6,000$ rounds.
  We set the arm set size $K=100$ and randomly select $K$ arms from $|\mathcal{A}|$ for the client.
  For the baseline conversational algorithms \texttt{ConUCB}, \texttt{Arm-Con}, and \texttt{ConLinUCB}, we adopt the conversation frequency function \(b(t)=5 \lfloor \log(t) \rfloor\), which adheres to their original papers.
  The result is shown in Figure~\ref{fig:regret-single-client}, where the \(x\)-axis is the number of rounds and the \(y\)-axis is the cumulative regret.
  We observe similar results among all the datasets, aligning with the findings from previous studies.
  That is, all the algorithms exhibit sublinear regrets with respect to the number of rounds.
  The algorithms without querying key terms (i.e., \texttt{LinUCB} and \texttt{Arm-Con}) have the poorest performance (largest cumulative regret), while other algorithms that query key terms perform better, thereby showing the importance of conversations about key terms.
  Our algorithm \fedconpe \emph{outperforms all of them}, achieving at least $5.25\%$ improvements among all the datasets.
  This performance improvement substantiates our theoretical results, showing that even under the non-federated scenario (i.e., \(M=1\)), our algorithm also achieves lower regret (see discussion in Remark~\ref{remark:regret}).

\begin{figure}[htb]
    \centering
    \includegraphics[width=\linewidth]{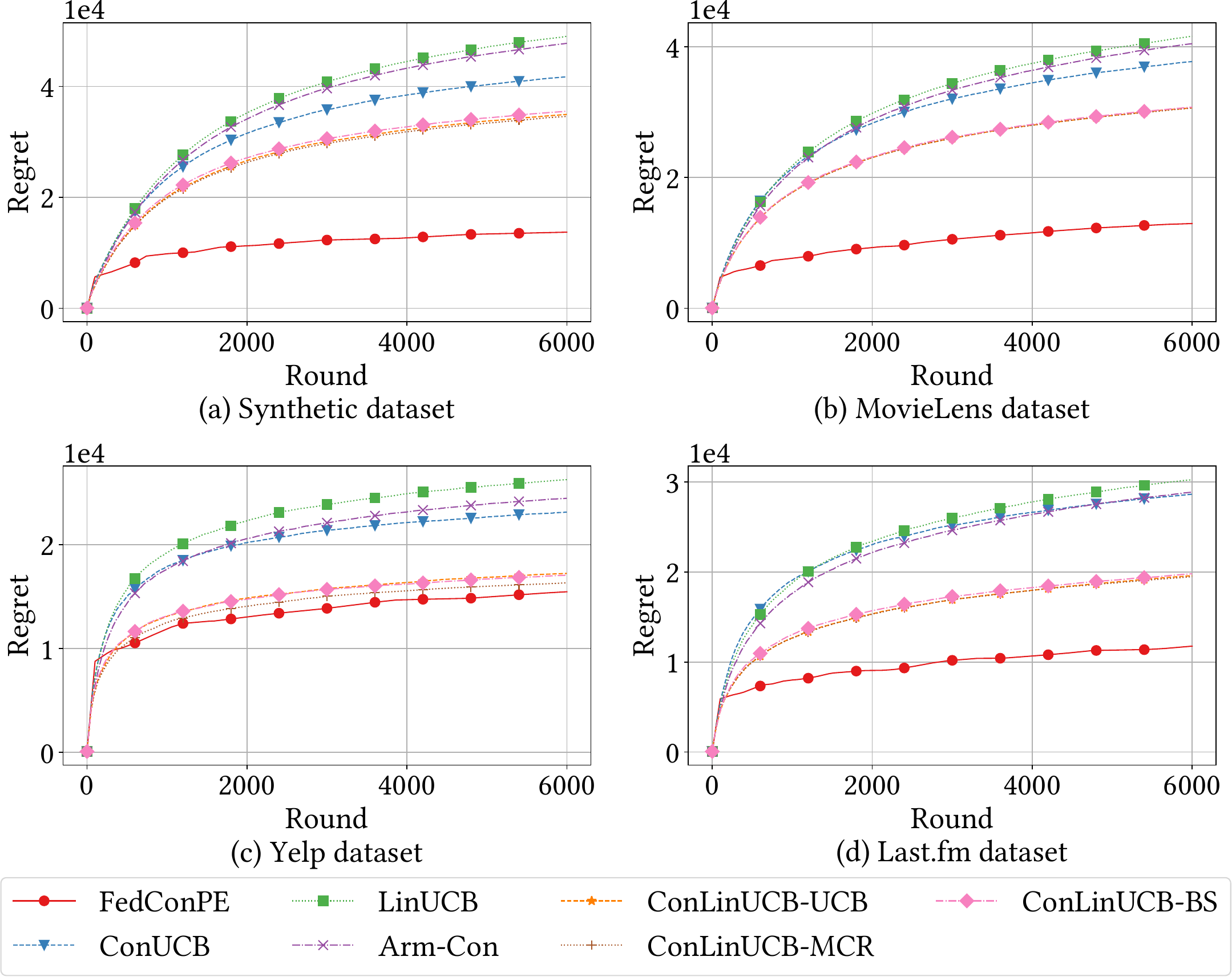}
    \caption{\label{fig:regret-single-client} Cumulative regret for the single-client scenario.}
  \end{figure}

\subsubsection{Cumulative Regret for Multiple Clients}
  Then, we evaluate the federated (i.e., multi-client) scenario.
  Since our work is the first to consider the federated setting in conversational bandits, when comparing with the baselines, we simply run the baseline algorithms on each client individually, without any communication with the server.
  We set the number of clients $M=10$ and independently select $K=100$ random arms for each client.
  Other parameters remain the same as the single-client setting.
  As depicted in Figure~\ref{fig:regret-multi-client}, \fedconpe shows even more advantages compared with the single-client scenario, achieving at least $37.05\%$ improvement over other algorithms.
  The advantage stems from the federated framework, where the central server aggregates data from each client to better estimate the unknown preference vector.

  \begin{figure}[htb]
    \centering
    \includegraphics[width=\linewidth]{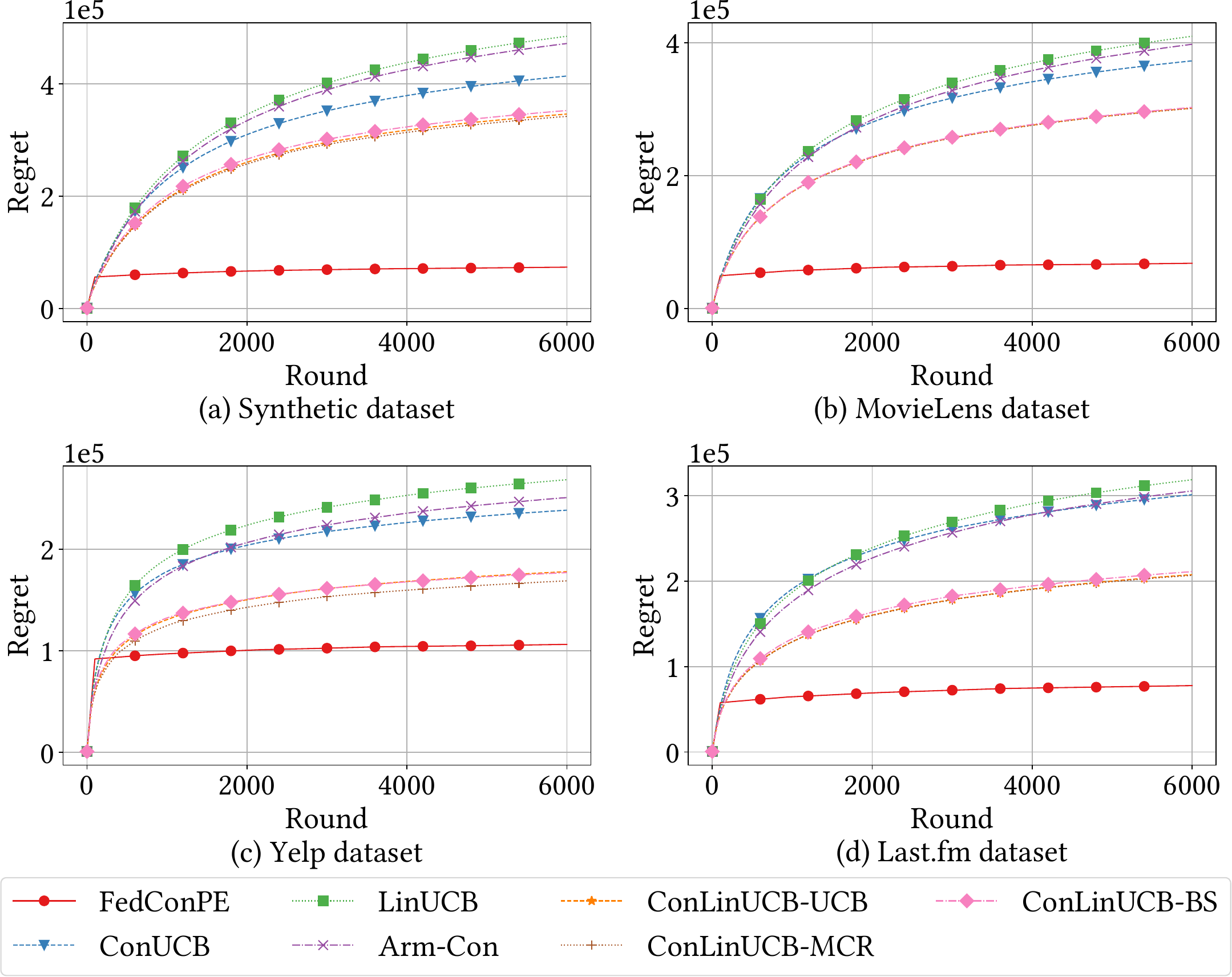}
    \caption{\label{fig:regret-multi-client} Cumulative regret for the multi-client scenario.}
  \end{figure}

  \subsubsection{Cumulative Regret with Different Number of Clients}
  We further demonstrate the effect of the number of clients by varying it from 3 to 15 with other parameters fixed and compare the cumulative regret at time \(T=6,000\).
  For illustration purposes, we only present the result for the synthetic dataset, while the results for the real-world datasets, which exhibit similar patterns, are provided in Appendix~\ref{appendix:plot-clients}.
  Figure~\ref{fig:regret-vs-num-clients} depicts that, without communication, the cumulative regrets of all the baseline algorithms increase linearly with the number of clients, i.e., \(\mathcal{\widetilde{O}}(dM\sqrt{T})\).
  In contrast, our algorithm leverages the data collected from all the clients, achieving a regret scaled in \(\mathcal{\widetilde{O}}(\sqrt{dMT})\), thereby mitigating the regret increase.

  \begin{figure}[htb]
    \includegraphics[width=\linewidth]{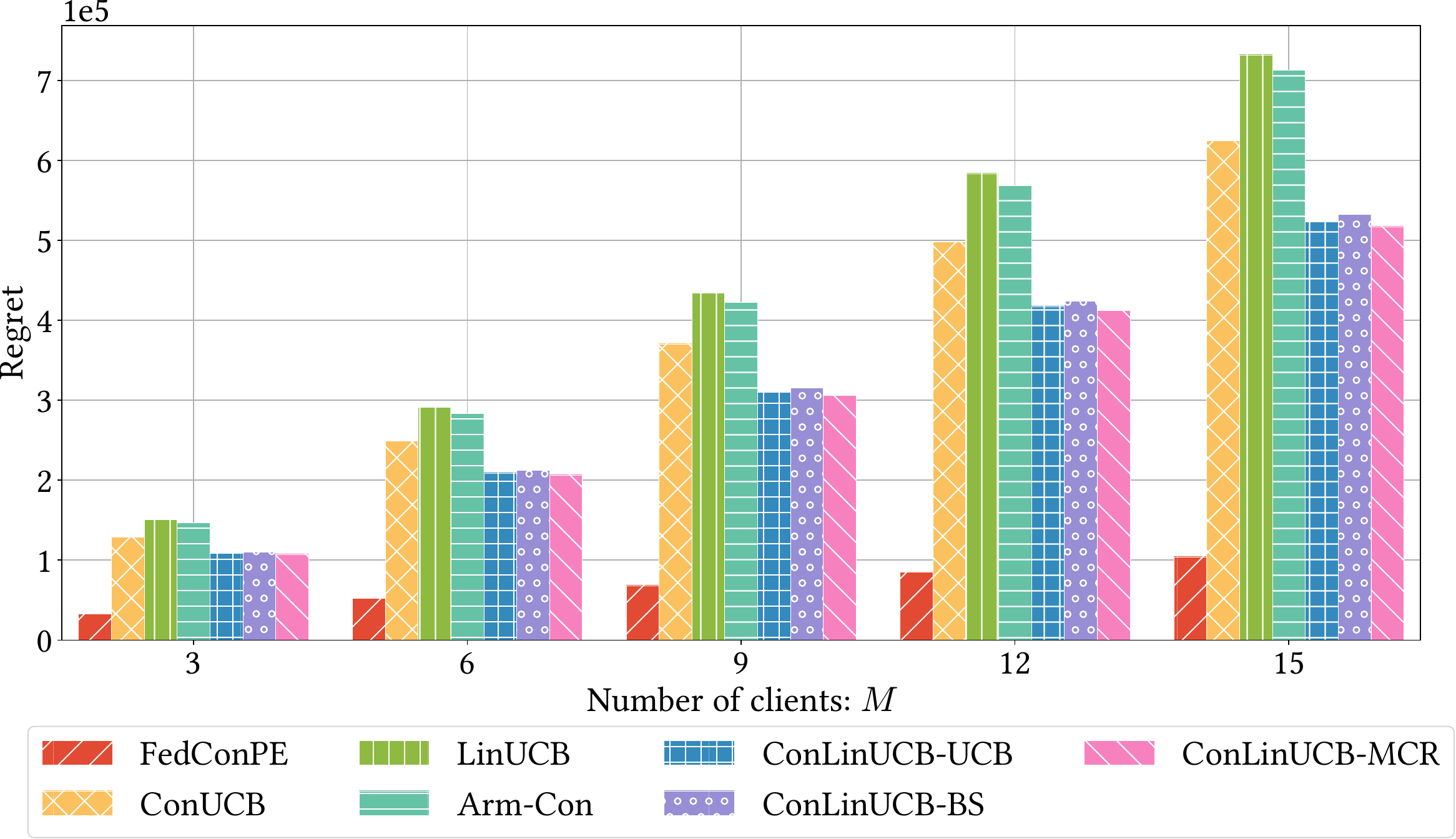}
    \caption{\label{fig:regret-vs-num-clients} Cumulative regret v.s. number of clients.}
  \end{figure}

  \subsubsection{Cumulative Regret with Different Arm Set Sizes}
  We evaluate the effect of the arm set size \(K\).
  Figure~\ref{fig:regret-vs-size-armset} illustrates the simulations executed for \(T=10,000\) under different arm set sizes, ranging from 100 to 300 on the synthetic dataset.
  Similar results derived from the real-world datasets are also provided in Appendix~\ref{appendix:plot-arms}.
  One can observe that, for all the baseline algorithms, there is no substantial amplification of the cumulative regrets as the arm set size \(K\) increases.
  This result is expected because all the baseline algorithms are based on the \texttt{LinUCB} framework, which is intrinsically designed to deal with infinite arm sets.
  Note that \fedconpe is also insensitive to the size \(K\).
  This validates our theoretical results (Theorem~\ref{thm:regret}), where the regret of our algorithm increases only logarithmically with respect to \(K\).

  \begin{figure}[htb]
    \includegraphics[width=\linewidth]{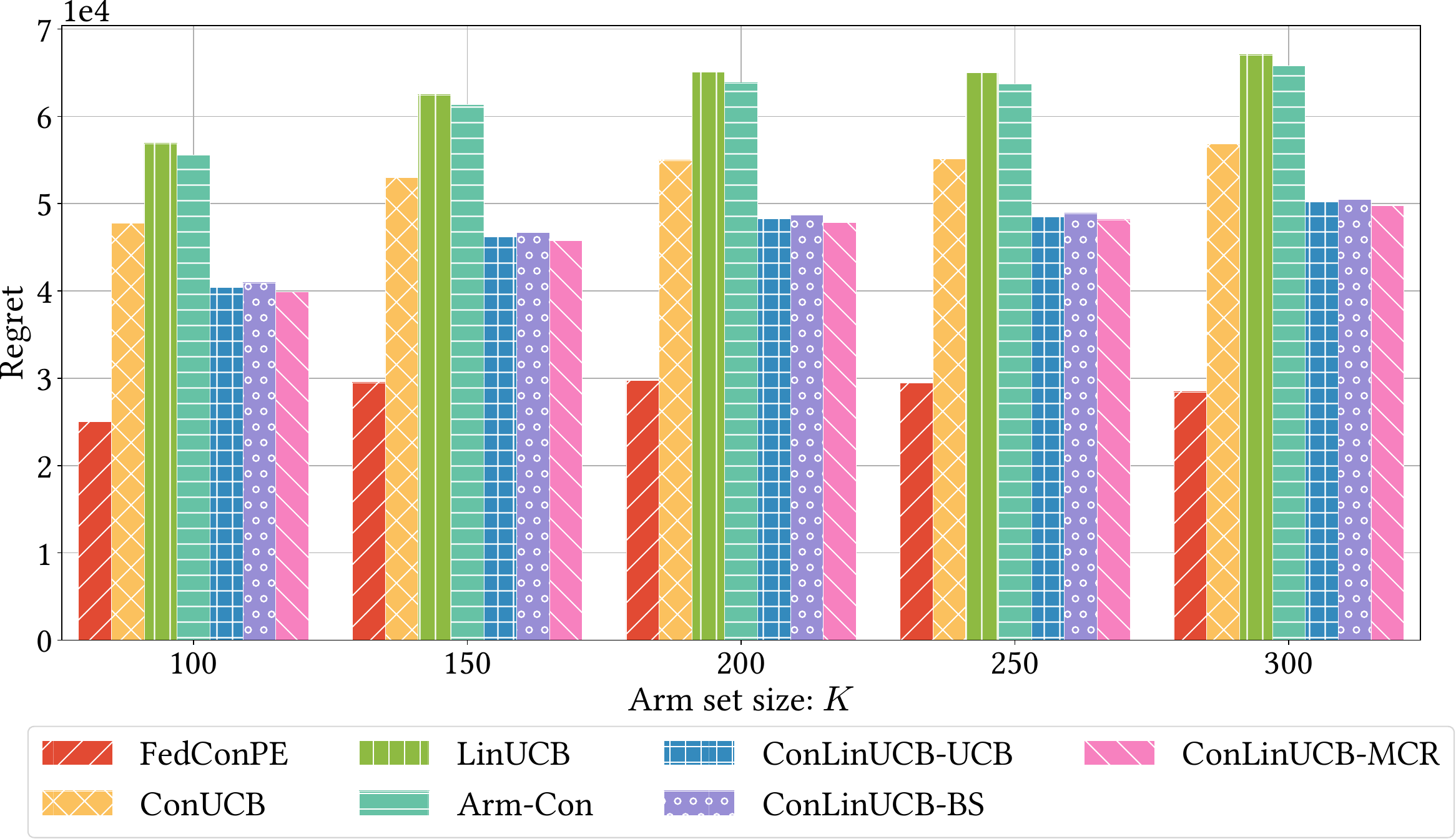}
    \caption{\label{fig:regret-vs-size-armset} Cumulative regret v.s. size of arm sets.}
  \end{figure}

\subsubsection{Accuracy of Estimated Preference Vectors}
  We evaluate the accuracy of the estimated preference vectors by computing the average \(\ell_2\)-distance between the estimated vector \(\widehat{\vec{\theta}}_{u,t}\) and the ground truth \(\vec{\theta}^{*}_{u}\) across 10 randomly selected users \(u\).
  The number of clients and the arm set size are configured to 5 and 100, respectively.
  In Figure~\ref{fig:theta-diff-vs-rounds}, we exhibit the average difference of all algorithms over the initial 2,000 rounds.
  Benefiting from the aggregated data, our algorithm \fedconpe can estimate the preference vector more quickly and accurately than the baseline algorithms.

  \begin{figure}[htb]
    \centering
    \includegraphics[width=\linewidth]{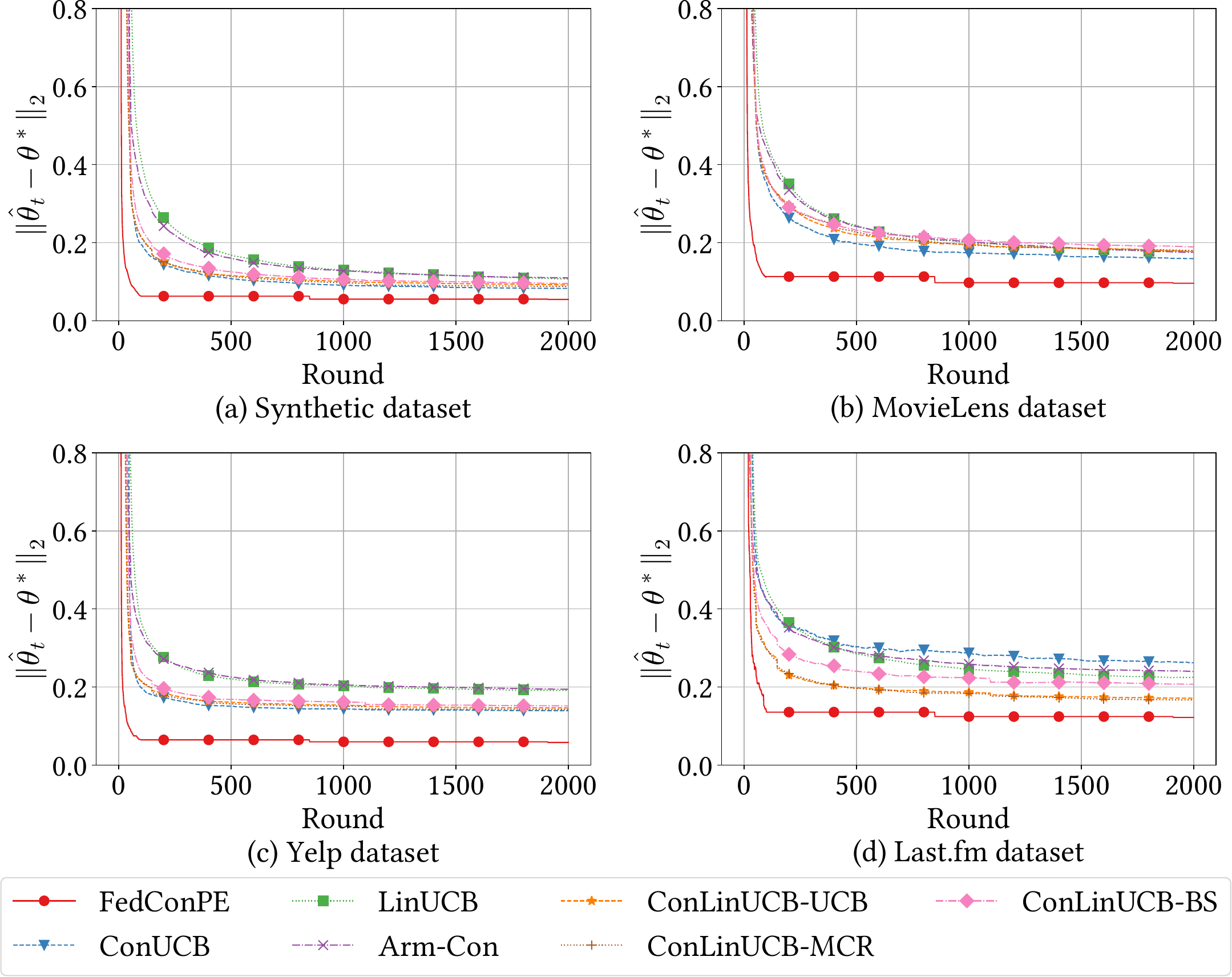}
    \caption{\label{fig:theta-diff-vs-rounds} Accuracy of estimated preference vectors.}
  \end{figure}

  \subsubsection{Number of Conversations}
  Finally, we measure the number of conversations (i.e., queries of key terms) initiated by each algorithm.
  Note that all of our baseline conversational algorithms (\texttt{ConUCB}, \texttt{Arm-Con}, and \texttt{ConLinUCB}) launch conversations in a deterministic manner, i.e., a pre-defined function \(b(t)\) governs the frequency of conversations.
  Therefore, we directly calculate the number of conversations launched according to their algorithms and plot the results.
  Following the original papers, we plot \(b(t)=5\lfloor\log(t)\rfloor\) and \(b(t)=\lfloor t/50\rfloor\), respectively.
  We note that although some baselines use a logarithmic \(b(t)\) in their experiments, they require a linear \(b(t)\) for their proofs.
  We run \fedconpe on all the datasets and record the number of key terms selected, with 10 users, and \(M=10\), \(K=100\).
  As shown in Figure~\ref{fig:keyterms-pulling-times}, due to the novel design of determining the need for conversations, \fedconpe launches fewer conversations than other algorithms, offering a better user experience.
  It is important to note that \fedconpe also provides enhanced flexibility concerning the order of conversations and recommendations within each phase (see details in Section~\ref{sec:client-algorithm}).

  \begin{figure}[htb]
    \centering
    \includegraphics[width=0.8\linewidth]{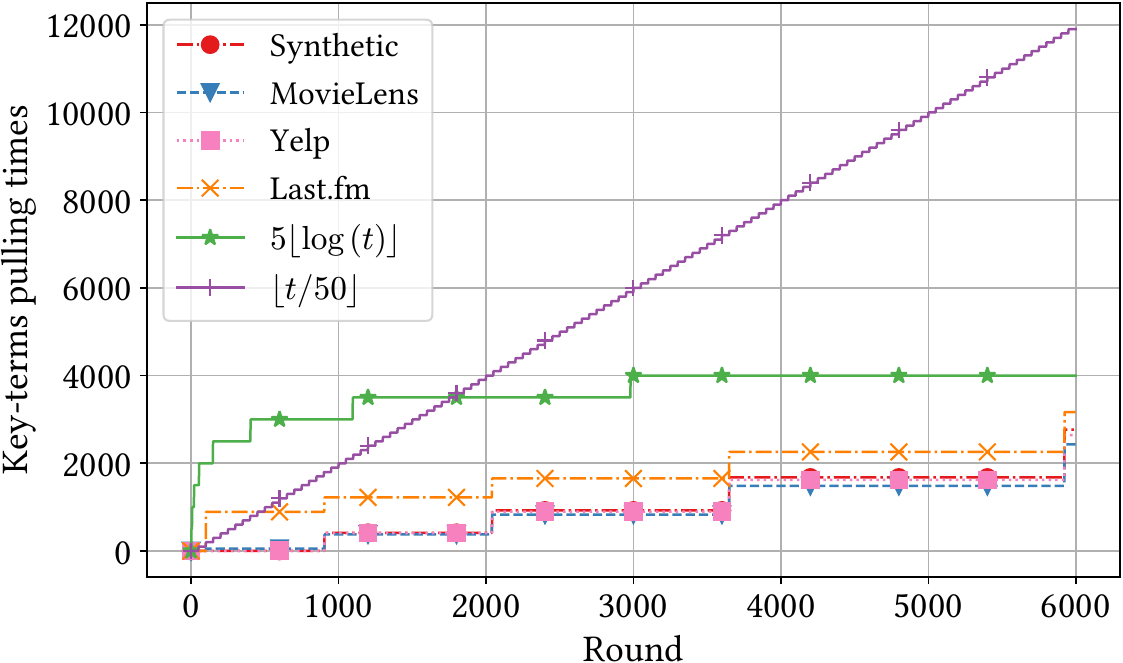}
    \caption{\label{fig:keyterms-pulling-times} The pulling times of key terms.}
  \end{figure}

\section{Conclusion}
\label{sec:conclusion}
  In this paper, we introduced \fedconpe, a phase elimination-based algorithm for federated conversational bandits with finite arm sets and heterogeneous clients.
  It adaptively constructs key terms that minimize uncertainty in the feature space.
  We proved that \fedconpe achieves matching regret lower\&upper bounds and has reduced communication cost, conversation frequency, and computational complexity.
  Extensive evaluations showed that our algorithm achieves a lower regret and requires fewer conversations than existing methods.

\section*{Acknowledgments}
We would like to thank the anonymous reviewers for their constructive comments.
The work of John C.S. Lui is supported in part by the RGC SRFS2122-4S02.

\section*{Contribution Statement}
Zhuohua Li and Maoli Liu contributed equally to this work.

\bibliographystyle{named}
\bibliography{ijcai24}

\appendix
\clearpage
\onecolumn
\input{appendix.tex}

\end{document}

%% file: algorithms/client.tex
\begin{algorithm}[htb]
  \DontPrintSemicolon
  \small
  \SetKwComment{Comment}{$\triangleright$\ }{}
  \SetKwInput{KwInit}{Initialization}
  \KwIn{Time horizon \(T\), number of clients \(M\), dimension \(d\), number of arms \(K\), arm set \(\mathcal{A}_i \subseteq \RR^d\), constant \(C \in (0,1]\) in Assumption~\ref{assumption:2}, constant \(N>0\)}
  \KwInit{Let \(\ell=1, \mathcal{A}_i^{\ell} = \mathcal{A}_i\)}
  \SetKwProg{Fn}{Function}{:}{}

  \While{not reaching time horizon \(T\)}{
    Compute G-optimal design \(\pi_i^{\ell}\) on \(\mathcal{A}_i^{\ell}\), such that\;
    \Indp\nonl\(\sum_{\vec{a} \in \mathcal{A}_i^{\ell}} \pi_i^{\ell}(\vec{a}) = 1, \ \ \vec{V}_i^{\ell}(\pi) = \sum_{\vec{a} \in \mathcal{A}_i^{\ell}} \pi(\vec{a}) \vec{a} \vec{a}^\mathsf{T}\)\;
    \nonl\(g(\pi_i^{\ell}) = \max_{\vec{a} \in \mathcal{A}_i^{\ell}} \|\vec{a}\|_{\vec{V}_i^{\ell}(\pi_i^{\ell})^{-1}}^{2}=d\)\;
    \Indm Diagonalize \(\vec{V}_i^{\ell}(\pi_i^{\ell})= \sum_{j=1}^{d} \lambda_{\vec{v}_j} \vec{v}_j \vec{v}_j^\mathsf{T}\) \label{line:diagonalization}\;
    \ForEach{\(\lambda_{\vec{v}_j} < \frac{3}{4(1-\varepsilon_{\ell}^2)dN}\)\label{line:check-eigenvalue}}{
      Upload \((\lambda_{\vec{v}_j}, \vec{v}_j)\)\;
    }
    Download \(\mathcal{K}_i^{\ell}\) and \(\set{n_{\vec{k}}}_{\vec{k} \in \mathcal{K}_i^{\ell}}\) from server\label{line:download-key-terms}\;
    \tcp{The order of plays is arbitrary}
    \ForEach(\Comment*[f]{Play arms}){\(\vec{a} \in \mathcal{A}_i^{\ell}\)}{
      Pull \(\vec{a}\) for \(T_{i,\ell}(\vec{a})=\left\lceil \frac{2d\pi_i^{\ell}(\vec{a})}{\varepsilon_{\ell}^2}\log \frac{2KM\log T}{\delta} \right\rceil\) times\;
      Receive rewards \(\set{x_{i,t}}_{t \in \mathcal{T}_{i,\vec{a}}^{\ell}}\)\;
    }
    \ForEach(\Comment*[f]{Play key terms}){\(\vec{k} \in \mathcal{K}_i^{\ell}\)}{
      Pull \(\vec{k}\) for \(n_{\vec{k}}\) times\;
      Receive rewards \(\set{\widetilde{x}_{i,t}}_{t \in \mathcal{\widetilde{T}}_{i,\vec{k}}^{\ell}}\)\;
    }
    \(
    \begin{aligned}
      \text{Upload }\vec{G}_i^{\ell} &= \sum_{\vec{a} \in \mathcal{A}_i^{\ell}} T_{i,\ell}(\vec{a}) \vec{a} \vec{a}^\mathsf{T} + \sum_{\vec{k} \in \mathcal{K}_i^{\ell}} n_{\vec{k}} \vec{k} \vec{k}^\mathsf{T}\text{, and}\\
      \vec{W}_i^{\ell} &= \sum_{t \in \mathcal{T}_i^{\ell}} \vec{a}_{i,t} x_{i,t} + \sum_{t \in \mathcal{\widetilde{T}}_i^{\ell}} \vec{k}_{i,t} \widetilde{x}_{i,t}
    \end{aligned}
    \)\label{line:upload-data}\;
    Download \(\widehat{\vec{\theta}}_{\ell}\) from server\;
    Eliminate suboptimal arms:
    \(\displaystyle \mathcal{A}_i^{\ell+1} = \set{\vec{a} \in \mathcal{A}_i^{\ell}: \max_{\vec{b} \in \mathcal{A}_i^{\ell}} \inprod{\widehat{\vec{\theta}}_{\ell}}{\vec{b}-\vec{a}} \leq 2\sqrt{\frac{N}{M}}\varepsilon_{\ell}}\) \label{line:elimination}\;
    \(\ell = \ell + 1\)\;
  }
  \caption{\texttt{FedConPE} Algorithm for client \(i\)} \label{algo:client}
\end{algorithm}

%% file: algorithms/server.tex
\begin{algorithm}[htb]
  \DontPrintSemicolon
  \small
  \SetKwInput{KwInit}{Initialization}
  \KwIn{Time horizon \(T\), number of clients \(M\), dimension \(d\), key term set \(\mathcal{K} \subseteq \RR^d\), constant \(C \in (0,1]\) in Assumption~\ref{assumption:2}, constant \(N>0\)}
  \KwInit{Let \(\ell=1, \vec{G}=\vec{0}, \vec{W}=\vec{0}\)}
  \SetKwFunction{SupportExploration}{SupportExploration}
  \SetKwProg{Fn}{Function}{:}{}

  \While{not reaching time horizon \(T\)}{
    \ForEach{\(i \in [M]\)}{
      Receive \(E_i=\set{(\lambda_{\vec{v}_j}, \vec{v}_j)}_j\), let \(\mathcal{K}_i^{\ell} = \emptyset\)\;
      \ForEach{\((\lambda_{\vec{v}_j}, \vec{v}_j) \in E_i\)}{
        \(\vec{k} = \argmax_{\vec{v} \in \mathcal{K}} \vec{v}^\mathsf{T}\vec{v}_j\)\label{line:find-key-term}\;
        \(\mathcal{K}_i^{\ell} = \mathcal{K}_i^{\ell} \cup (\lambda_{\vec{v}_j}, \vec{k})\) \;
        \(n_{\vec{k}} = \left\lceil \frac{\frac{3}{2(1-\varepsilon_{\ell}^2)N}-2d\lambda_{\vec{v}_j}}{C^2\varepsilon_{\ell}^2}\log \frac{2KM\log T}{\delta} \right\rceil\)\;
      }
      Send \(\mathcal{K}_i^{\ell}\) and \(\set{n_{\vec{k}}}_{\vec{k} \in \mathcal{K}_i^{\ell}}\) to client \(i\)\label{line:send-back-key-terms}\;
      Receive \(\vec{G}_i^{\ell}\) and \(\vec{W}_i^{\ell}\)\;
    }
    \(\vec{G} = \sum_{p \in [\ell]} \sum_{i \in [M]} \vec{G}_i^p,\ \
    \vec{W} = \sum_{p \in [\ell]}\sum_{i \in [M]} \vec{W}_i^p\) \label{line:aggregate-data}\;
    Broadcast \(\widehat{\vec{\theta}}_{\ell} = \vec{G}^{-1} \vec{W}\) to all clients \label{line:estimate-theta}\;
    \(\ell = \ell + 1\)\;
  }
  \caption{\texttt{FedConPE} Algorithm for server} \label{algo:server}
\end{algorithm}

%% file: appendix.tex
\section{Summary of Notations}\label{sec:notations}
Table~\ref{tab:notation} summarizes the notations used through the main paper and the appendix.

\begin{table}[htb]
\centering
\caption{Table of Notations}
\label{tab:notation}
\begin{tabular}{ccl}
\toprule
Symbol &  & Meaning \\
\midrule
$C$ & & Constant introduced in Assumption~\ref{assumption:2}.\\
$N$ & & Constant introduced in Algorithm~\ref{algo:client} and Algorithm~\ref{algo:server}.\\
$M$ & & Number of clients.\\
$d$ & & Dimension.\\
$K$ & & Number of candidate items\slash key terms.\\
$T$ & & Number of rounds.\\
$L$ & & Number of phases.\\
$t$ & & Index of rounds.\\
$p,\ell$ & & Index of phases.\\
$\varepsilon_{\ell}$ & & Defined as $\varepsilon_{\ell}=2^{-\ell}$.\\
$\mathcal{A}_i$ & & Candidate arm set for client $i$.\\
$\mathcal{A}_i^{\ell}$ & & Active arm set for client $i$ in phase $\ell$.\\
$\mathcal{K}$ & & Candidate key term set.\\
$\vec{\theta}^*, \widetilde{\vec{\theta}}^*$ & & Arm-level and key term-level preference vector, respectively.\\
$\widehat{\vec{\theta}}_{\ell}$ & & Estimated preference vector in phase $\ell$.\\
$\vec{a}_i^*$ & & The optimal arm for client $i$.\\
$\vec{a}_{i,t}, \vec{k}_{i,t}$ & & Arm and key term selected by client $i$ at round $t$, respectively.\\
$\eta_{i,t}, \widetilde{\eta}_{i,t}$ & & Noise term of arm-level and key term-level feedback for client $i$ at round $t$, respectively.\\
$x_{i,t}, \widetilde{x}_{i,t}$ & & Arm-level and key term-level feedback received by client $i$ at round $t$, respectively.\\
$\mathcal{T}_{i,\vec{a}}^{\ell}, \widetilde{\mathcal{T}}_{i,\vec{k}}^{\ell}$ & & Set of rounds during which client $i$ plays arm $\vec{a}$ and key term $\vec{k}$, respectively.\\
$\mathcal{T}_{i}^{\ell}, \widetilde{\mathcal{T}}_{i}^{\ell}$ & & Set of rounds during which client $i$ plays arms and key terms, respectively.\\
$\pi_i^{\ell}$ & & G-optimal design computed by client $i$ in phase $\ell$.\\
$\vec{V}_i^{\ell}(\pi_i^{\ell})$ & & Information matrix of G-optimal design computed by client $i$ in phase $\ell$.\\
$\lambda_{\vec{v}}$ & & Eigenvalue associated with eigenvector $\vec{v}$.\\
$\lambda_{\text{min}}(\vec{M})$ & & The smallest eigenvalue of matrix $\vec{M}$.\\
\(\lambda_k(\vec{M})\) & & The \(k\)-th ordered eigenvalue of matrix $\vec{M}$.\\
$T_{i,\ell}(\vec{a})$ & & Number of times that arm $\vec{a}$ is selected by client $i$ in phase $\ell$.\\
$\widetilde{T}_{i,\ell}$ & & Number of key terms selected in phase $\ell$ by client $i$.\\
$T_{\ell}$ & & Number of arms selected in phase $\ell$ by a client.\\
$n_{\vec{k}}$ & & Number of times that key term $\vec{k}$ will be selected.\\
$\vec{G}_i^{\ell}$, $\vec{W}_i^{\ell}$ & & $d \times d$ Gram matrix and $d \times 1$ moment vector collected by client $i$ in phase $\ell$, respectively.\\
$\vec{G}$, $\vec{W}$ & & $d \times d$ Gram matrix and $d \times 1$ moment vector collected by the server, respectively.\\
$N_i(t)$, $\widetilde{N}_j(t)$ & & Number of times the \(i\)-th arm and the \(j\)-th key term are chosen, respectively after the end of round \(t\).\\
\(N_{i,a}(t)\) & & Number of times the \(a\)-th arm is chosen by client \(i\) after the end of round \(t\).\\
$\mathcal{H}_t$ & & Sequence of outcomes generated by the interaction between the policy and the environment until round \(t\).\\
$\pi^{\text{arm}}$, $\pi^{\text{key}}$ & & Policies for selecting arms and key terms, respectively.\\
\(P_{i}\), \(\widetilde{P}_{j}\) & & Arm-level and key term-level reward distributions for the $i$-th arm and $j$-th key term, respectively.\\
\(p_{i}\), \(\widetilde{p}_{j}\) & & Density functions of arm-level and key term-level reward distributions \(P_{i}\) and \(\widetilde{P}_{j}\), respectively.\\
\bottomrule
\end{tabular}
\end{table}

\section{Preliminaries}
\label{sec:preliminaries}
  We introduce some well-known results that will be used in Appendix~\ref{sec:proof-regret} and Appendix~\ref{sec:proof-lowerbound} without proofs.

  \begin{lemma}[Subgaussian random variables]\label{lemma:subgaussian}
    Suppose that random variables \(X\) is \(\sigma\)-subgaussian, \(X_1\) and \(X_2\) are independent and \(\sigma_1\) and \(\sigma_2\)-subgaussian, respectively, then
    \begin{enumerate}
      \item For any \(\varepsilon>0\), \(\Pr\left[X\geq \varepsilon\right] \leq \exp\left(-\frac{\varepsilon^2}{2\sigma^2}\right)\).
      \item \(X_1 + X_2\) is \(\sqrt{\sigma_1^2 + \sigma_2^2}\)-subgaussian.
    \end{enumerate}
  \end{lemma}

  \begin{lemma}[\protect\citeauthor{bretagnolle-huber-1978}]\label{lemma:bretagnolle-huber}
    Let \(P\) and \(Q\) be probability measures on the same measurable space \((\Omega, \mathcal{F})\), and let \(A \in \mathcal{F}\) be an arbitrary event. Then,
    \[P(A) + Q(A^c) \geq \frac{1}{2} \exp(-D(P \parallel Q)),\]
    where \(D(P \parallel Q)=\int_{\Omega}\log\left(\odv{P}{Q}\right)\odif{P} = \E_P\left[\log\odv{P}{Q}\right]\) is the KL divergence between \(P\) and \(Q\). \(A^c = \Omega \setminus A\) is the complement of \(A\).
  \end{lemma}

  \begin{lemma}[KL divergence between Gaussian distributions]\label{lemma:kl-divergence-of-gaussian}
    If \(P \sim \mathcal{N}(\mu_1, \sigma^2)\) and \(Q \sim \mathcal{N}(\mu_2, \sigma^2)\), then
    \[D(P \parallel Q)=\frac{(\mu_1-\mu_2)^2}{2\sigma^2}.\]
  \end{lemma}

\section{Proof of Theorem~\ref{thm:regret} (Regret Upper Bound)}
\label{sec:proof-regret}
  We first prepare several lemmas:

\begin{lemma}[Concentration of linear regression]\label{lemma:concentration}
  For any \(\delta > 0\), \(\ell \in [L]\), \(\vec{x} \in \RR^d\), with probability at least \(1-2\delta\), we have
  \[\left| \inprod{\widehat{\vec{\theta}}_{\ell} - \vec{\theta}^{*}}{\vec{x}} \right| \leq \sqrt{2 \|\vec{x}\|_{\vec{G}^{-1}}^2 \log \frac{1}{\delta}}.\]
\end{lemma}
\begin{proof}
  We only prove one side, the other side is similar.
  \begin{align*}
    &\inprod{\widehat{\vec{\theta}}_{\ell} - \vec{\theta}^{*}}{\vec{x}} = \inprod{\vec{x}}{\vec{G}^{-1} \vec{W} - \vec{\theta}^{*}}\\
    =& \inprod{\vec{x}}{\vec{G}^{-1} \sum_{p=1}^{\ell}\sum_{i=1}^{M}\Big(\sum_{\mathclap{t \in \mathcal{T}_i^{p}}} \vec{a}_{i,t}x_{i,t} + \sum_{\mathclap{t \in \mathcal{\widetilde{T}}_i^{p}}} \vec{k}_{i,t}\widetilde{x}_{i,t}\Big)-\vec{\theta}^{*} }\\
    =& \inprod{\vec{x}}{\vec{G}^{-1} \sum_{p=1}^{\ell}\sum_{i=1}^{M}\ab[\sum_{t \in \mathcal{T}_i^p} \vec{a}_{i,t} \ab( \vec{a}_{i,t}^\mathsf{T} \vec{\theta}^{*} + \eta_{i,t}) + \sum_{t \in \mathcal{\widetilde{T}}_i^p} \vec{k}_{i,t} \ab( \vec{k}_{i,t}^\mathsf{T} \vec{\theta}^{*} + \widetilde{\eta}_{i,t})] -\vec{\theta}^{*} }\\
    =& \inprod{\vec{x}}{\vec{G}^{-1} \underbrace{\sum_{p=1}^{\ell}\sum_{i=1}^{M}\bigg(\sum_{t \in \mathcal{T}_i^{p}} \vec{a}_{i,t}\vec{a}_{i,t}^\mathsf{T} + \sum_{t \in \mathcal{\widetilde{T}}_i^{p}}\vec{k}_{i,t}\vec{k}_{i,t}^\mathsf{T}\bigg)}_{=\vec{G}} \vec{\theta}^{*} + \vec{G}^{-1}\sum_{p=1}^{\ell}\sum_{i=1}^{M}\bigg(\sum_{t \in \mathcal{T}_i^{p}} \vec{a}_{i,t}\eta_{i,t}+\sum_{t \in \mathcal{\widetilde{T}}_i^{p}} \vec{k}_{i,t}\widetilde{\eta}_{i,t}\bigg) -\vec{\theta}^{*} }\\
    =& \inprod{\vec{x}}{\vec{G}^{-1} \sum_{p=1}^{\ell}\sum_{i=1}^{M}\bigg(\sum_{t \in \mathcal{T}_i^{p}} \vec{a}_{i,t}\eta_{i,t}+\sum_{t \in \mathcal{\widetilde{T}}_i^{p}} \vec{k}_{i,t}\widetilde{\eta}_{i,t}\bigg)}\\
    =& \sum_{p=1}^{\ell}\sum_{i=1}^{M}\Big(\sum_{\mathclap{t \in \mathcal{T}_i^{p}}}\inprod{\vec{x}}{\vec{G}^{-1} \vec{a}_{i,t}}\eta_{i,t}+\sum_{\mathclap{t \in \mathcal{\widetilde{T}}_i^{p}}}\inprod{\vec{x}}{\vec{G}^{-1} \vec{k}_{i,t}}\widetilde{\eta}_{i,t}\Big). \numberthis \label{eq:sum-of-subgaussian}
  \end{align*}
  With some basic linear algebra, we can show that
  \begin{align*}
    &\sum_{p=1}^{\ell}\sum_{i=1}^{M}\bigg(\sum_{t \in \mathcal{T}_i^{p}}\inprod{\vec{x}}{\vec{G}^{-1} \vec{a}_{i,t}}^2 + \sum_{t \in \mathcal{\widetilde{T}}_i^{p}}\inprod{\vec{x}}{\vec{G}^{-1} \vec{k}_{i,t}}^2\bigg)\\
    =&\sum_{p=1}^{\ell}\sum_{i=1}^{M}\bigg(\sum_{t \in \mathcal{T}_i^{p}} \vec{x}^\mathsf{T}\vec{G}^{-1}\vec{a}_{i,t}\vec{a}_{i,t}^\mathsf{T}\vec{G}^{-1}\vec{x} + \sum_{t \in \mathcal{\widetilde{T}}_i^{p}} \vec{x}^\mathsf{T}\vec{G}^{-1}\vec{k}_{i,t}\vec{k}_{i,t}^\mathsf{T}\vec{G}^{-1}\vec{x} \bigg)\\
    =&\vec{x}^\mathsf{T} \vec{G}^{-1}\underbrace{\sum_{p=1}^{\ell}\sum_{i=1}^{M}\bigg(\sum_{t \in \mathcal{T}_i^{p}} \vec{a}_{i,t}\vec{a}_{i,t}^\mathsf{T} + \sum_{t \in \mathcal{\widetilde{T}}_i^{p}}\vec{k}_{i,t}\vec{k}_{i,t}^\mathsf{T}\bigg)}_{=\vec{G}} \vec{G}^{-1} \vec{x}\\
    =& \|\vec{x}\|_{\vec{G}^{-1}}^{2}.
  \end{align*}
  Therefore, since \(\eta_{i,t}\) and \(\widetilde{\eta}_{i,t}\) are independent and 1-subgaussian, Equation~\ref{eq:sum-of-subgaussian} is \(\|\vec{x}\|_{\vec{G}^{-1}}\)-subgaussian by Lemma~\ref{lemma:subgaussian}.
  We have for all \(\varepsilon \geq 0\),
  \[\Pr\left[\inprod{\widehat{\vec{\theta}}_{\ell} - \vec{\theta}^{*}}{x} \geq \varepsilon\right] \leq \exp\left(-\frac{\varepsilon^2}{2\|\vec{x}\|_{\vec{G}^{-1}}^{2}}\right).\]
  Let the right-hand-side be \(\delta\), we get
  \[\Pr\left[\inprod{\widehat{\vec{\theta}}_{\ell} - \vec{\theta}^{*}}{\vec{x}} \geq \sqrt{2\|\vec{x}\|_{\vec{G}^{-1}}^{2} \log \frac{1}{\delta}}\right] \leq \delta.\]
  The other side is similar.
  The result follows by a union bound.
\end{proof}

\begin{lemma}[Bound the number of phases]\label{lemma:bound-num-of-phases}
  Given the time horizon T, the total number of phases \(L \leq \log T\).
\end{lemma}
\begin{proof}
  According to Algorithm~\ref{algo:client}, in each phase \(\ell\), the total number of plays
  \[T_{\ell} = \sum_{\vec{a} \in \mathcal{A}_{i}^{\ell}} T_{i,\ell}(\vec{a})=\sum_{\vec{a} \in \mathcal{A}_{i}^{\ell}}\left\lceil \frac{2d\pi_i^{\ell}(\vec{a})}{\varepsilon_{\ell}^2}\log \frac{2KM\log T}{\delta} \right\rceil \geq \frac{2d}{\varepsilon_{\ell}^2} \log \frac{2KM\log T}{\delta},\]
  where we use the fact that \(\sum_{\vec{a} \in \mathcal{A}_{i}^{\ell}} \pi_i^{\ell}(\vec{a}) = 1\). Therefore,
  \begin{align*}
    T &\geq \sum_{\ell=1}^{L} T_{\ell} \geq \sum_{\ell=1}^{L} \frac{2d}{\varepsilon_{\ell}^2} \log \frac{2KM\log T}{\delta}\\
      &= 2d \frac{4}{3}(2^{2L}-1)\log \frac{2KM\log T}{\delta}\\
      &\geq 2d 2^{2L} \log \frac{KM\log T}{\delta} \numberthis \label{eq:bound-num-of-phases}\\
      &\geq 2^L.
  \end{align*}
  The last inequality is due to \(2d 2^L \log \frac{KM\log T}{\delta} \geq 1\).
\end{proof}

\begin{lemma}[Eigenvalue of information matrix with conversational information]\label{lemma:lower-bound-of-smallest-eigenvalue}
  For any client \(i \in [M]\) and any phase \(\ell \in [L]\), we have
  \[\lambda_{\text{min}}\ab(\vec{V}_i^{\ell}(\pi_i^{\ell}) + \sum_{(\lambda,\vec{k}) \in \mathcal{K}_i^{\ell}} \frac{\frac{3}{4(1-\varepsilon_{\ell}^2)dN}-\lambda}{C^2}\vec{k} \vec{k}^\mathsf{T}) \geq \frac{3}{4(1-\varepsilon_{\ell}^2)dN},\]
  where \(\lambda_{\text{min}}\) denotes the smallest eigenvalue, \(\mathcal{K}_i^{\ell}\) contains the eigenvalue-eigenvector pairs that the server sends to client \(i\) in phase \(\ell\), and \(0<C\leq 1\) is the constant defined in Assumption~\ref{assumption:2}.
\end{lemma}
\begin{proof}
  Let \(\vec{v}_1, \dots, \vec{v}_d\) and \(\lambda_1, \dots, \lambda_d\) be the eigenvectors and the corresponding eigenvalues of \(\vec{V}_i^{\ell}(\pi_i^{\ell})\).
  Using the eigenvectors as an orthonormal basis, for any \(j \in [d]\), we can write any vector \(\vec{k} = \sum_{i=1}^{d} c_i \vec{v}_i = \sum_{i=1, i\neq j}^d c_i \vec{v}_i + c_j \vec{v}_j\).
  Note that \(\vec{x}\triangleq \sum_{i=1, i\neq j}^d c_i \vec{v}_i\) is orthogonal to \(\vec{v}_j\).
  According to Line~\ref{line:find-key-term} of Algorithm~\ref{algo:server}, every \(\vec{k}\) in \(\mathcal{K}_i^{\ell}\) is chosen to be close to an eigenvector, say \(\vec{v}_j\).
  Then by Assumption~\ref{assumption:2}, we have \(\vec{k}^\mathsf{T} \vec{v}_j\geq C\).
  Therefore, we have \((\sum_{i=1}^{d} c_i \vec{v}_i)^\mathsf{T} \vec{v}_j = c_j \geq C\), and \(\vec{k} \vec{k}^\mathsf{T} = (c_j \vec{v}_j+\vec{x})(c_j \vec{v}_j+\vec{x})^\mathsf{T}=c_j^2 \vec{v}_j \vec{v}_j^\mathsf{T} + \vec{x} \vec{x}^\mathsf{T}\).
  Let \(s_{\ell} = \frac{3}{4(1-\varepsilon_{\ell}^2)dN}\), we have,
  \begin{align*}
    &\vec{V}_i^{\ell}(\pi_i^{\ell}) + \sum_{(\lambda,\vec{k}) \in \mathcal{K}_i^{\ell}} \frac{\frac{3}{4(1-\varepsilon_{\ell}^2)dN}-\lambda}{C^2}\vec{k} \vec{k}^\mathsf{T}\\
    =& \sum_{i=1}^{d} \lambda_i \vec{v}_i \vec{v}_i^\mathsf{T} + \sum_{j: \lambda_j<s_{\ell}} \frac{s_{\ell}-\lambda_j}{C^2} \ab(c_j^2 \vec{v}_j \vec{v}_j^\mathsf{T} + \vec{x} \vec{x}^\mathsf{T}) \numberthis\label{eq:expand-and-change-sum-condition}\\
    \succeq& \sum_{i=1}^{d} \lambda_i \vec{v}_i \vec{v}_i^\mathsf{T} + \sum_{j: \lambda_j<s_{\ell}} \ab(s_{\ell} - \lambda_j) \vec{v}_j \vec{v}_j^\mathsf{T}\\
    =& \sum_{j: \lambda_j<s_{\ell}} s_{\ell} \vec{v}_j \vec{v}_j^\mathsf{T} + \sum_{j: \lambda_j\geq s_{\ell}} \lambda_j \vec{v}_j \vec{v}_j^\mathsf{T}\\
    \succeq& \sum_{i=1}^{d} \frac{3}{4(1-\varepsilon_{\ell}^2)dN} \vec{v}_i \vec{v}_i^\mathsf{T}.
  \end{align*}
  Where in Equation~\ref{eq:expand-and-change-sum-condition} we expand \(\vec{V}_i^{\ell}(\pi_i^{\ell})\) by its spectral decomposition and change the summation condition of the second term according to Algorithm~\ref{algo:client}, Line~\ref{line:check-eigenvalue}.
  The proof concludes by the property of Loewner order that if \(\vec{A} \succeq \vec{B}\) then \(\lambda_k(\vec{A}) \geq \lambda_k(\vec{B})\), where \(\lambda_k\) denotes the \(k\)th ordered eigenvalue.
\end{proof}

\begin{lemma}[Bound the probability of a ``bad'' event]\label{lemma:bad-event}
  Define the ``bad'' event:
  \[\mathcal{E} = \set{\exists i \in [M], \ell \in [L], \vec{a} \in \mathcal{A}_i^{\ell}, \text{ s.t. } \left| \inprod{\widehat{\vec{\theta}}_{\ell} - \vec{\theta}^{*}}{\vec{a}} \right| > \sqrt{\frac{N}{M}}\varepsilon_{\ell}}.\]
  Then \(\Pr\left[\mathcal{E}\right] \leq \delta\).
\end{lemma}
\begin{proof}
  Let \(s_{\ell} = \frac{3}{4(1-\varepsilon_{\ell}^2)dN}\). For any phase \(\ell \in [L]\), according to the definition of \(\vec{G}\) in Algorithm~\ref{algo:server}, we have
  \begin{align*}
    \vec{G} &= \sum_{p=1}^{\ell} \sum_{i=1}^{M} \vec{G}_i^p\\
            &=\sum_{p=1}^{\ell}\sum_{i=1}^{M}\left[\sum_{\vec{a} \in \mathcal{A}_i^{p}} T_{p,i}(\vec{a}) \vec{a}\vec{a}^\mathsf{T} + \sum_{(\lambda,\vec{k}) \in \mathcal{K}_i^{p}} n_{\vec{k}} \vec{k} \vec{k}^\mathsf{T}\right] \\
            &\succeq \sum_{p=1}^{\ell}\sum_{i=1}^{M} \left[\sum_{\vec{a} \in \mathcal{A}_i^{p}}\frac{2d\pi_i^{p}(\vec{a})}{\varepsilon_{p}^2} \log\ab(\frac{2KM\log T}{\delta}) \vec{a}\vec{a}^\mathsf{T} + \sum_{(\lambda,\vec{k}) \in \mathcal{K}_i^{p}}\frac{2d(s_p-\lambda)}{C^2\varepsilon_{p}^2} \log\ab(\frac{2KM\log T}{\delta}) \vec{k} \vec{k}^\mathsf{T} \right]\\
            &= 2d \log \ab(\frac{2KM\log T}{\delta}) \sum_{i=1}^{M} \underbrace{\ab[\sum_{p=1}^{\ell} \frac{1}{\varepsilon_{p}^2}\sum_{\vec{a} \in \mathcal{A}_i^{p}} \pi_i^{p}(\vec{a}) \vec{a}\vec{a}^\mathsf{T} + \sum_{p=1}^{\ell} \frac{1}{\varepsilon_{p}^2} \sum_{(\lambda,\vec{k}) \in \mathcal{K}_i^{p}} \frac{s_p - \lambda}{C^2}\vec{k} \vec{k}^\mathsf{T}]}_{\triangleq \vec{Q}_i^{\ell}}.
  \end{align*}

  Next, we prove the lower bound of the smallest eigenvalue of \(\vec{Q}_i^{\ell}\).
  By Weyl's inequality~\cite{weyl-1912}, we have \(\lambda_{\text{min}}(\vec{A}+\vec{B}) \geq \lambda_{\text{min}}(\vec{A})+\lambda_{\text{min}}(\vec{B})\), therefore
  \begin{align*}
    \lambda_{\text{min}}\ab(\vec{Q}_i^{\ell}) &= \lambda_{\text{min}}\ab(\sum_{p=1}^{\ell} \frac{1}{\varepsilon_{p}^2}\ab[\sum_{\vec{a} \in \mathcal{A}_i^{p}} \pi_i^{p}(\vec{a}) \vec{a}\vec{a}^\mathsf{T} + \sum_{(\lambda,\vec{k}) \in \mathcal{K}_i^{p}} \frac{s_p - \lambda}{C^2}\vec{k} \vec{k}^\mathsf{T}])\\
    &\geq \sum_{p=1}^{\ell} \frac{1}{\varepsilon_{p}^2} \lambda_{\text{min}}\ab(\vec{V}_i^{p}(\pi_i^{p}) + \sum_{(\lambda,\vec{k}) \in \mathcal{K}_i^{p}} \frac{s_p - \lambda}{C^2}\vec{k} \vec{k}^\mathsf{T})\\
    &\geq \sum_{p=1}^{\ell} \frac{1}{\varepsilon_{p}^2} \frac{3}{4(1-\varepsilon_{p}^2)dN}
      \geq \frac{3}{4(1-\varepsilon_{\ell}^2)dN}\sum_{p=1}^{\ell} \frac{1}{\varepsilon_{p}^2} \numberthis\label{eq:use-smallest-eigenvalue-lemma}\\
    &=\frac{3}{4(1-\varepsilon_{\ell}^2)dN} \cdot \frac{4}{3}\ab(\frac{1}{\varepsilon_{\ell}^2} - 1) = \frac{1}{dN\cdot \varepsilon_{\ell}^2},
  \end{align*}
  where Equation~\ref{eq:use-smallest-eigenvalue-lemma} uses Lemma~\ref{lemma:lower-bound-of-smallest-eigenvalue} and the fact that \(\varepsilon_{\ell}=2^{-\ell}\) is decreasing in \(\ell\).
  Based on this, we can further get the lower bound of the smallest eigenvalue of \(\vec{G}\) as follows.
  \begin{align*}
    \lambda_{\text{min}} \left( \vec{G} \right) &\geq 2d \log \ab(\frac{2KM\log T}{\delta})\sum_{i=1}^{M} \lambda_{\text{min}} \left( \vec{Q}_i^{\ell} \right)\\
    &\geq \frac{2M}{N\varepsilon_{\ell}^2} \log \frac{2KM\log T}{\delta}.
  \end{align*}
  By Lemma~\ref{lemma:concentration}, with probability at least \(1-\frac{\delta}{KM\log T}\), we have for all client \(i \in [M]\) and all arm \(\vec{a} \in \mathcal{A}_i^{\ell}\),
  \begin{align*}
    \left| \inprod{\widehat{\vec{\theta}}_{\ell} - \vec{\theta}^{*}}{\vec{a}} \right|
    &\leq \sqrt{2\|\vec{a}\|_{\vec{G}^{-1}}^{2} \log \frac{2KM\log T}{\delta}}\\
    &\leq \sqrt{2 \frac{1}{\lambda_{\text{min}}(\vec{G})} \log \frac{2KM\log T}{\delta}} \numberthis \label{eq:introduce-lambda-min}\\
    &\leq \sqrt{\frac{N}{M}}\varepsilon_{\ell}.
  \end{align*}
  Where Equation~\ref{eq:introduce-lambda-min} is due to Courant-Fischer theorem.
  Finally, by union bound,
  \begin{align*}
    \Pr\left[\mathcal{E}\right] \leq& \sum_{i \in [M]}\sum_{\ell \in [L]}\sum_{\vec{a} \in \mathcal{A}_i^{\ell}} \Pr\left[\left| \inprod{\widehat{\vec{\theta}}_{\ell} - \vec{\theta}^{*}}{\vec{a}} \right| > \sqrt{\frac{N}{M}}\varepsilon_{\ell}\right]\\
    \leq& MLK \frac{\delta}{KM\log T} \leq \delta.
  \end{align*}
  Where the last inequality uses the fact that \(L\leq \log T\), which is due to Lemma~\ref{lemma:bound-num-of-phases}.
\end{proof}

\begin{lemma}\label{lemma:optimal-arm-never-eliminated}
  Under the ``good'' event \(\mathcal{E}^c\), for any \(i \in [M], \ell \in [L]\), we have \(\vec{a}_i^{*} \in \mathcal{A}_i^{\ell}\), where \(\vec{a}_i^{*} \triangleq \argmax_{\vec{a} \in \mathcal{A}_i} \inprod{\vec{a}}{\vec{\theta}^{*}}\), i.e., the local optimal arm will never be eliminated.
\end{lemma}
\begin{proof}
  If \(\mathcal{E}^c\) happens, then for any phase \(\ell \in [L]\), client \(i \in [M]\), arm \(\vec{b} \in \mathcal{A}_i^{\ell}\), we have
  \begin{align*}
    &\inprod{\vec{b}-\vec{a}_i^{*}}{\widehat{\vec{\theta}}_{\ell}}\\
    =& \inprod{\vec{b}}{\widehat{\vec{\theta}}_{\ell}} - \inprod{\vec{b}}{\vec{\theta}^{*}} + \inprod{\vec{b}}{\vec{\theta}^{*}} - \inprod{\vec{a}_i^{*}}{\widehat{\vec{\theta}}_{\ell}} + \inprod{\vec{a}_i^{*}}{\vec{\theta}^{*}} - \inprod{\vec{a}_i^{*}}{\vec{\theta}^{*}}\\
    =&\inprod{\vec{b}}{\widehat{\vec{\theta}}_{\ell}-\vec{\theta}^{*}} - \inprod{\vec{a}_i^{*}}{\widehat{\vec{\theta}}_{\ell}-\vec{\theta}^{*}} + \inprod{\vec{b}-\vec{a}_i^{*}}{\vec{\theta}^{*}}\\
    \leq& \sqrt{\frac{N}{M}}\varepsilon_{\ell} + \sqrt{\frac{N}{M}}\varepsilon_{\ell} + 0 = 2\sqrt{\frac{N}{M}}\varepsilon_{\ell}.
  \end{align*}
  Where the last inequality uses the ``good'' event and the fact that \(\vec{a}_i^{*}\) is locally optimal.
\end{proof}

\begin{lemma}
  Under the ``good'' event \(\mathcal{E}^c\), for any phase \(\ell \in [L]\), client \(i \in [M]\), arm \(\vec{a} \in \mathcal{A}_i^{\ell}\), we have
  \[\Delta_{i,\vec{a}} \triangleq \inprod{\vec{\theta}^{*}}{\vec{a}_i^{*} - \vec{a}} \leq 8\sqrt{\frac{N}{M}}\varepsilon_{\ell}.\]
\end{lemma}
\begin{proof}
  If \(\mathcal{E}^c\) happens, then by Lemma~\ref{lemma:bad-event} and Lemma~\ref{lemma:optimal-arm-never-eliminated}, consider phase \(\ell-1 \in [L]\), for any \(i \in [M], \vec{a} \in \mathcal{A}_i^{\ell-1}\), we have \(\inprod{\widehat{\vec{\theta}}_{\ell-1}}{\vec{a}} \leq \inprod{\vec{\theta}^{*}}{\vec{a}} + \sqrt{\frac{N}{M}}\varepsilon_{\ell-1}\) and \(\inprod{\widehat{\vec{\theta}}_{\ell-1}}{\vec{a}_i^{*}} \geq \inprod{\vec{\theta}^{*}}{\vec{a}_i^{*}} - \sqrt{\frac{N}{M}}\varepsilon_{\ell-1}\).
  One event that eliminate arm \(\vec{a}\) is:
  \[\inprod{\widehat{\vec{\theta}}_{\ell-1}}{\vec{a}_i^{*}} -\sqrt{\frac{N}{M}}\varepsilon_{\ell-1} \geq \inprod{\widehat{\vec{\theta}}_{\ell-1}}{\vec{a}} + \sqrt{\frac{N}{M}}\varepsilon_{\ell-1}.\]
  This event is guaranteed to happen (so arm \(\vec{a}\) is eliminated) as long as \(\inprod{\vec{\theta}^{*}}{\vec{a}_i^{*}} - 2\sqrt{N/M}\varepsilon_{\ell-1} > \inprod{\vec{\theta}^{*}}{\vec{a}} + 2\sqrt{N/M}\varepsilon_{\ell-1}\), which we can rearrange as \(\Delta_{i,a} \triangleq \inprod{\vec{\theta}^{*}}{\vec{a}_i^{*} - \vec{a}} > 4\sqrt{N/M}\varepsilon_{\ell-1}=8\sqrt{N/M}\varepsilon_{\ell}\).
  Consider the contrapositive, if arm \(\vec{a}\) is not eliminated, i.e., \(\vec{a} \in \mathcal{A}_i^{\ell}\), then we have \(\Delta_{i,\vec{a}} \leq 8\sqrt{N/M}\varepsilon_{\ell}\).
\end{proof}

Finally, we can prove Theorem~\ref{thm:regret}, which we restate here.

\restateregret*

\begin{proof}
  Assume the total number of phases is \(L\), with probability \(1-\delta\), the cumulative regret can be bounded as:
  \begin{align*}
    R_{M}(T) =&\sum_{i=1}^{M}\sum_{t=1}^{T} \inprod{\vec{\theta}^{*}}{\vec{a}_i^{*}-\vec{a}_{i,t}}
        \leq \sum_{i=1}^{M}\sum_{\ell=1}^{L} T_{\ell} \cdot 8\sqrt{\frac{N}{M}}\varepsilon_{\ell}\\
        =&8\sqrt{N}\sum_{i=1}^{M} \sum_{\ell=1}^{L} \sum_{\vec{a} \in \mathcal{A}_i^{\ell}} \left\lceil \frac{2d\pi_i^{\ell}(\vec{a})}{\varepsilon_{\ell}^2} \log \frac{KM\log T}{\delta} \right\rceil \frac{\varepsilon_{\ell}}{\sqrt{M}}\\
        \leq& 8\sqrt{NM}\sum_{\ell=1}^{L} \left( \frac{2d}{\varepsilon_{\ell}^2} \log \frac{KM\log T}{\delta} + \frac{d(d+1)}{2}\right) \varepsilon_{\ell} \numberthis \label{eq:remove-ceil}\\
        =& 16d\sqrt{NM}\log \frac{KM\log T}{\delta} \sum_{\ell=1}^{L} 2^{\ell} + 4d(d+1)\sqrt{NM} \sum_{\ell=1}^{L} \frac{1}{2^{\ell}}\\
        \leq& 32d\sqrt{NM}\log \frac{KM\log T}{\delta} 2^L + 4d(d+1)\sqrt{NM}\\
        \leq& 32\sqrt{N} \sqrt{dMT \log \frac{KM\log T}{\delta}} +4d(d+1)\sqrt{NM} \numberthis \label{eq:introduce-T}\\
        =& \mathcal{O}\left( \sqrt{dMT\log \frac{KM \log T}{\delta}} \right). \numberthis \label{eq:remove-small-terms}
  \end{align*}
  where in Equation~\ref{eq:remove-ceil} we remove the ceiling according to Lemma~\ref{lemma:kiefer-wolfowitz}.
  In Equation~\ref{eq:introduce-T} we replace \(2^L\) according to Equation~\ref{eq:bound-num-of-phases}.
  And in Equation~\ref{eq:remove-small-terms} we ignore constant and insignificant terms because \(T\gg M\) and \(T\gg d\).
\end{proof}

\section{Proof of Theorem~\ref{thm:lowerbound} (Regret Lower Bound)}
\label{sec:proof-lowerbound}

  Since all clients face the same linear bandit problem, any algorithm in the federated setting with \(M\) clients over time horizon \(T\) can be simulated by an algorithm under the centralized setting with total \(MT\) observations available for learning.
  As a result, the cumulative regret of \(M\) clients in the federated setting over horizon \(T\) must be no better than a single-client linear bandit over horizon \(MT\).
  Therefore, in the following, we consider the latter setting.
  we emphasize that for federated linear bandits with fixed arm sets, one cannot directly apply the well-known result \(\Omega(\sqrt{dT})\) by \cite{chu-2011-contextual} because it depends on the property of time-varying arm sets.
  In addition, the main challenge of deriving the lower bound is to handle the conversational information.

  Since algorithms for conversational bandits need to make decisions on selecting both arms and key terms, we consider that a policy \(\pi\) consists of two parts \(\pi = (\pi^{\text{arm}}, \pi^{\text{key}})\).
  We also assume that at each time step, the policy can select at most one key term. (Otherwise the number of key terms may be larger than the number of arms, which is unrealistic.)
  Since in our setting, we assume that for all client \(i\), arms in \(\mathcal{A}_i\) span \(\RR^d\), so the number of arms \(K\geq d\).
  Let \(\mathcal{A}_1=\mathcal{A}_2=\dots=\mathcal{A}_M = \mathcal{K}= \set{\vec{e}_1, \vec{e}_2, \dots, \vec{e}_d} \cup \set{(K-d)\text{ arbitrary unit vectors}}\), where \(\vec{e}_i\) is the \(i\)-th standard basis vector in \(\RR^d\).
  We choose \(\vec{\theta}=(\Delta,0,\dots,0)^\mathsf{T}\) (\(\Delta\) will be specified later.) and denote \(\mathcal{H}_t = \set{a_1, x_1, k_1, \widetilde{x}_1, \dots,a_{t}, x_{t}, k_{t}, \widetilde{x}_{t}}\) as a sequence of outcomes generated by the interaction between the policy and the environment until time \(t\).
  We note that the appearance of key terms at every time step is without loss of generality because we allow \(k_t\) to be empty if there is no conversation initiated at round \(t\).
  The noise term of both arm and key term level feedback (\(\eta_{i,t}\) and \(\widetilde{\eta}_{i,t}\)) are independently drawn from a standard Gaussian distribution \(\mathcal{N}(0,1)\).
  We denote \(\PP_{\vec{\theta}}\) as the probability measure induced by the environment \(\vec{\theta}\) and policy \(\pi\), and expectations under \(\PP_{\vec{\theta}}\) is denoted by \(\E_{\vec{\theta}}\).
  Let random variables \(N_i(t)\), \(\widetilde{N}_j(t)\) be the number of times the \(i\)-th arm and the \(j\)-th key term (\(i,j \in [K]\)) are chosen, respectively after the end of round \(t\).
  Define another environment \(\vec{\theta}' = (\Delta, 0, \dots, \underbrace{2\Delta}_{s\text{-th}}, \dots, 0)^\mathsf{T}\), where \(s=\argmin_{j>1} \max\ab\{\E_{\vec{\theta}}[N_j(MT)],\E_{\vec{\theta}}[\widetilde{N}_j(MT)]\}\).
  Given the definitions, it is easy to see that \(\E_{\vec{\theta}}[N_j(MT)] \leq \frac{MT}{K-1}\) and \(\E_{\vec{\theta}}[\widetilde{N}_j(MT)] \leq \frac{MT}{K-1}\).
  Therefore, arm\slash key term 1 is the optimal arm\slash key term for all clients under environment \(\vec{\theta}\), while under \(\vec{\theta}'\), the optimal arm\slash key term is arm\slash key term \(s\).
  Based on the above settings, we prove the following lemma.

  \begin{lemma}\label{lemma:divergence-decomposition}
    Denote \(D(P \parallel Q)=\int_{\Omega}\log\left(\odv{P}{Q}\right)\odif{P}\) as is the KL divergence between \(P\) and \(Q\), then we have
    \[D(\PP_{\vec{\theta}} \parallel \PP_{\vec{\theta}'}) = \E\nolimits_{\vec{\theta}}[N_s(MT) + \widetilde{N}_s(MT)] D(\mathcal{N}(0,1) \parallel \mathcal{N}(2\Delta,1)).\]
  \end{lemma}

  \begin{proof}
    Given a bandit instance with parameter \(\vec{\theta}\) and a policy \(\pi\), according to Section 4.6 of \cite{lattimore-2020-bandit-algorithms}, we construct the canonical bandit model of our setting as follows.
    Let \((\Omega,\mathcal{F}, \PP_{\vec{\theta}})\) be a probability space where \(\Omega=([K] \times \RR)^{MT}\), \(\mathcal{F}=\mathcal{B}(\Omega)\), and the density function of the probability measure \(\PP_{\vec{\theta}}\) is defined by \(p_{\vec{\theta}}: \Omega \to \RR\):
    \begin{align*}
      p_{\vec{\theta}}(\mathcal{H}_{MT}) = \prod_{t=1}^{MT} \pi_t^{\text{arm}}(a_t \mid \mathcal{H}_{t-1}) p_{a_t}(x_t) \cdot \pi_t^{\text{key}}(k_t \mid \mathcal{H}_{t-1}) \widetilde{p}_{k_t}(\widetilde{x}_t),
    \end{align*}
    where \(p_{a_t}\) and \(\widetilde{p}_{k_t}\) are the density functions of arm-level and key term-level reward distributions \(P_{a_t}\) and \(\widetilde{P}_{k_t}\), respectively.
    The definition of \(\PP_{\vec{\theta}'}\) is identical except that \(p_{a_t}\), \(\widetilde{p}_{k_t}\) are replaced by \(p'_{a_t}\), \(\widetilde{p}'_{k_t}\) and \(P_{a_t}\), \(\widetilde{P}_{k_t}\) are replaced by \(P'_{a_t}\), \(\widetilde{P}'_{k_t}\).

    By the definition of KL divergence, we have
    \begin{align*}
      D(\PP_{\vec{\theta}} \parallel \PP_{\vec{\theta}'}) =\int_{\Omega}\log\left(\odv{\PP_{\vec{\theta}}}{\PP_{\vec{\theta}'}}\right)\odif{\PP_{\vec{\theta}}}= \E\nolimits_{\vec{\theta}}\left[ \log \odv{\PP_{\vec{\theta}}}{\PP_{\vec{\theta}'}} \right].
    \end{align*}
    Note that
    \begin{align*}
      &\log \ab(\odv{\PP_{\vec{\theta}}}{\PP_{\vec{\theta}'}}(\mathcal{H}_T))= \log \frac{p_{\vec{\theta},\pi}(\mathcal{H}_{T})}{p_{\vec{\theta}',\pi}(\mathcal{H}_{T})} \numberthis\label{eq:chain-rule-radon-nikodym}\\
      =& \log \frac{\prod_{t=1}^{MT} \pi_t^{\text{arm}}(a_t \mid \mathcal{H}_{t-1}) p_{a_t}(x_t) \cdot \pi_t^{\text{key}}(k_t \mid \mathcal{H}_{t-1}) \widetilde{p}_{k_t}(\widetilde{x}_t)}{\prod_{t=1}^{MT} \pi_t^{\text{arm}}(a_t \mid \mathcal{H}_{t-1}) p'_{a_t}(x_t) \cdot \pi_t^{\text{key}}(k_t \mid \mathcal{H}_{t-1}) \widetilde{p}'_{k_t}(\widetilde{x}_t)}\\
      =& \sum_{t=1}^{MT} \ab(\log \frac{p_{a_t}(x_t)}{p'_{a_t}(x_t)} + \log \frac{\widetilde{p}_{k_t}(\widetilde{x}_t)}{\widetilde{p}'_{k_t}(\widetilde{x}_t)}).
    \end{align*}
    where in Equation~\ref{eq:chain-rule-radon-nikodym} we used the chain rule for Radon–Nikodym derivatives, and in the last equality, all the terms involving the policy cancel.
    Therefore,
    \begin{align*}
      D(\PP_{\vec{\theta}} \parallel \PP_{\vec{\theta}'})
      =& \sum_{t=1}^{MT} \ab(\E\nolimits_{\vec{\theta}} \ab[\log \frac{p_{A_t}(X_t)}{p'_{A_t}(X_t)}] + \E\nolimits_{\vec{\theta}}\ab[\log \frac{\widetilde{p}_{K_t}(\widetilde{X}_t)}{\widetilde{p}'_{K_t}(\widetilde{X}_t)}])\\
      =& \sum_{t=1}^{MT} \ab(\E\nolimits_{\vec{\theta}}\ab[\E\nolimits_{\vec{\theta}} \ab[\log \frac{p_{A_t}(X_t)}{p'_{A_t}(X_t)} \bigg\mid A_t]] + \E\nolimits_{\vec{\theta}}\ab[\E\nolimits_{\vec{\theta}}\ab[\log \frac{\widetilde{p}_{K_t}(\widetilde{X}_t)}{\widetilde{p}'_{K_t}(\widetilde{X}_t)}\bigg\mid K_t]])\\
      =& \sum_{t=1}^{MT} \ab(\E\nolimits_{\vec{\theta}}[D(P_{A_t} \parallel P'_{A_t})] + \E\nolimits_{\vec{\theta}}[D(\widetilde{P}_{K_t} \parallel \widetilde{P}'_{K_t})])\\
      =& \sum_{i=1}^{K} \E\nolimits_{\vec{\theta}} \ab[\sum_{t=1}^{MT} \ab(\1_{\set{A_t=i}}D(P_{A_t} \parallel P'_{A_t}) + \1_{\set{K_t=i}}D(\widetilde{P}_{K_t} \parallel \widetilde{P}'_{K_t}))]\\
      =& \sum_{i=1}^{K} \E\nolimits_{\vec{\theta}}[N_i(MT) + \widetilde{N}_i(MT)]D(P_i \parallel P'_i)\\
      =& \E\nolimits_{\vec{\theta}}[N_s(MT) + \widetilde{N}_s(MT)]D(\mathcal{N}(0,1) \parallel \mathcal{N}(2\Delta,1)), \numberthis\label{eq:divergence-differs-only-at-s}
    \end{align*}
    where Equation~\ref{eq:divergence-differs-only-at-s} is due to the fact that \(P_i\) and \(P'_i\) only differs when \(i=s\).
  \end{proof}

  \restatelowerbound*
  \begin{proof}
    Define \(R_{M,\vec{\theta}}^{\pi}(T)\) as the expected cumulative regret of policy \(\pi\) under environment \(\vec{\theta}\) over \(M\) clients and time horizon \(T\).
    Denote \(N_{i,a}(t)\) as the number of times the \(a\)-th arm (\(a \in [K]\)) is chosen by client \(i \in [M]\) after the end of round \(t\), then
    \begin{align*}
      R_{M,\vec{\theta}}^{\pi}(T)
      &= \E\nolimits_{\vec{\theta}}\ab[\sum_{i=1}^{M} \sum_{t=1}^{T} \left( \max_{\vec{a} \in \mathcal{A}_i} \vec{a}^\mathsf{T} \vec{\theta}^{*} - \vec{a}_{i,t}^\mathsf{T} \vec{\theta}^{*} \right)]\\
      &= \sum_{i=1}^{M} \sum_{t=1}^{T} \E\nolimits_{\vec{\theta}}\ab[\Delta_{i,\vec{a}_{i,t}}]\\
      &= \sum_{i=1}^{M} \Delta \sum_{a=2}^{K} \E\nolimits_{\vec{\theta}}[N_{i,a}(T)] \numberthis\label{eq:introduce-N}\\
      &= \sum_{i=1}^{M} \Delta (T-\E\nolimits_{\vec{\theta}}[N_{i,1}(T)])\\
      &= \Delta\ab(MT-\E\nolimits_{\vec{\theta}}\ab[\sum_{i=1}^{M} N_{i,1}(T)])\\
      &\geq \Delta \Pr\nolimits_{\vec{\theta}}\ab[MT-\sum_{i=1}^{M} N_{i,1}(T) \geq \frac{MT}{2}] \frac{MT}{2},
    \end{align*}
    where Equation~\ref{eq:introduce-N} is because the optimal arm for environment \(\vec{\theta}\) is arm 1, so pulling other arms increments the expected regret by \(\Delta\).
    The last inequality is due to Markov inequality.
    For environment \(\vec{\theta}'\), similarly we have
    \begin{align*}
      R_{M,\vec{\theta}'}^{\pi}(T)
      &= \sum_{i=1}^{M} \sum_{a=1}^{K} \Delta_{i,a} \E\nolimits_{\vec{\theta}'}[N_{i,a}(T)]\\
      &\geq \sum_{i=1}^{M} \Delta \E\nolimits_{\vec{\theta}'}[N_{i,1}(T)]
      = \Delta\E\nolimits_{\vec{\theta}'}\ab[\sum_{i=1}^{M}N_{i,1}(T)] \numberthis\label{eq:only-maintain-arm1-ignore-others}\\
      &\geq \Delta \Pr\nolimits_{\vec{\theta}'}\ab[\sum_{i=1}^{M} N_{i,1}(T) > \frac{MT}{2}] \frac{MT}{2},
    \end{align*}
    where in Equation~\ref{eq:only-maintain-arm1-ignore-others} we only keep the term for arm 1.
    Therefore,
    \begin{align*}
      R_{M,\vec{\theta}}^{\pi}(T) + R_{M,\vec{\theta}'}^{\pi}(T)
      &\geq \frac{\Delta MT}{2} \ab(\Pr\nolimits_{\vec{\theta}}\ab[\sum_{i=1}^{M} N_{i,1}(T) \leq \frac{MT}{2}] + \Pr\nolimits_{\vec{\theta}'}\ab[\sum_{i=1}^{M} N_{i,1}(T) > \frac{MT}{2}])\\
      &\geq \frac{\Delta MT}{4} \exp\ab(-D(\PP_{\theta} \parallel \PP_{\theta'})) \numberthis\label{eq:use-bretagnolle-huber}\\
      &\geq \frac{\Delta MT}{4} \exp\ab(-\E\nolimits_{\vec{\theta}}[N_s(MT) + \widetilde{N}_s(MT)] D(\mathcal{N}(0,1) \parallel \mathcal{N}(2\Delta,1))) \numberthis\label{eq:use-divergence-decomposition}\\
      &\geq \frac{\Delta MT}{4} \exp\ab(-\frac{2MT}{K-1} \frac{(2\Delta)^2}{2}) \numberthis\label{eq:use-gaussian-KL-divergence}\\
      &\geq \frac{\Delta MT}{4} \exp\ab(-\frac{2MT}{d-1} \frac{(2\Delta)^2}{2}) \numberthis\label{eq:use-K-larger-than-d}
    \end{align*}
    where Equation~\ref{eq:use-bretagnolle-huber} is due to Bretagnolle-Huber theorem (Lemma~\ref{lemma:bretagnolle-huber}).
    Equation~\ref{eq:use-divergence-decomposition} uses Lemma~\ref{lemma:divergence-decomposition}.
    Equation~\ref{eq:use-gaussian-KL-divergence} uses Lemma~\ref{lemma:kl-divergence-of-gaussian} and the definition of \(s\).
    And Equation~\ref{eq:use-K-larger-than-d} is due to \(K\geq d\).

    Finally, let \(\Delta=\sqrt{\frac{d-1}{MT}}\), by using \(\max\ab\{a,b\}\geq \frac{a+b}{2}\), we have
    \begin{align*}
      \max\ab\{R_{M,\vec{\theta}}^{\pi}(T), R_{M,\vec{\theta}'}^{\pi}(T)\} \geq \frac{\exp(-4)}{8} \sqrt{(d-1)MT} = \Omega\ab(\sqrt{dMT}).
    \end{align*}

  \end{proof}

\section{Proof of Theorem~\ref{thm:communication} (Communication Cost)}
\label{sec:proof-communication}
\restatecommunication*
\begin{proof}
  At each phase \(\ell\), each client \(i\) downloads: (1) key terms set \(\mathcal{K}_i^{\ell}\), which contains at most \(d\) feature vectors with dimension \(d\), (2) repetition times of each key term \(\set{n_{\vec{k}}}_{\vec{k} \in \mathcal{K}_i^{\ell}}\), which are at most $d$ integers, and (3) estimated preference vector \(\widehat{\vec{\theta}}_{\ell}\), which is a \(d\)-dimensional vector.
  The client also uploads: (1) at most \(d\) eigenvalues and eigenvectors, and (2) matrix \(\vec{G}_i^{\ell}\) and \(\vec{W}_i^{\ell}\), which are both of size \(d^2\).
  Also note that by Lemma~\ref{lemma:bound-num-of-phases}, the number of phases is at most \(\log T\).
  Therefore, the upload cost and the download cost are both \(\mathcal{O}(d^2M \log T)\).
\end{proof}

\section{Proof of Theorem~\ref{thm:conversation} (Conversation Frequency Upper Bound)}
\label{sec:proof-conversation}
\restateconversation*
\begin{proof}
  (a) is straightforward from Algorithm~\ref{algo:client}, Line~\ref{line:check-eigenvalue}.
  For (b), in phase \(\ell \in [L]\), the number of arms pulled by each client is
  \[T_{\ell} = \sum_{\vec{a} \in \mathcal{A}_{i}^{\ell}} T_{i,\ell}(\vec{a})=\sum_{\vec{a} \in \mathcal{A}_{i}^{\ell}}\left\lceil \frac{2d\pi_i^{\ell}(\vec{a})}{\varepsilon_{\ell}^2}\log \frac{2KM\log T}{\delta} \right\rceil \geq \frac{2d}{\varepsilon_{\ell}^2} \log \frac{2KM\log T}{\delta}.\]
  While the number of key terms pulled by client \(i\) is
  \begin{align*}
    \widetilde{T}_{i,\ell}
    &=\sum_{(\lambda,\vec{k}) \in \mathcal{K}_i^{\ell}} n_{\vec{k}}\\
    &=\sum_{(\lambda,\vec{k}) \in \mathcal{K}_i^{\ell}}\frac{2d\ab(s_{\ell}-\lambda)}{C^2\varepsilon_{\ell}^2} \log\ab(\frac{2KM\log T}{\delta})\\
    &=\sum_{j:\lambda_j<s_{\ell}}\frac{2d\ab(s_{\ell}-\lambda_j)}{C^2\varepsilon_{\ell}^2} \log\ab(\frac{2KM\log T}{\delta})\\
    &\leq \sum_{j=1}^d\frac{2d\ab(s_{\ell}-\beta)}{C^2\varepsilon_{\ell}^2} \log\ab(\frac{2KM\log T}{\delta})\\
    &=\frac{2d\ab(\frac{3}{4(1-\varepsilon_{\ell}^2)N}-d\beta)}{C^2\varepsilon_{\ell}^2} \log\ab(\frac{2KM\log T}{\delta}),
  \end{align*}
  where \(s_{\ell} = \frac{3}{4(1-\varepsilon_{\ell}^2)dN}\).
  So the fraction between key terms and arms is upper bounded by
  \begin{align*}
    \frac{\widetilde{T}_{i,\ell}}{T_{\ell}} \leq \frac{\frac{3}{4(1-\varepsilon_{\ell}^2)N}-d\beta}{C^2}=\frac{\frac{3}{4(1-\varepsilon_{\ell}^2)}-dN\beta}{NC^2} \leq \frac{1-dN\beta}{NC^2}.
  \end{align*}
\end{proof}

\section{Data Generation and Preprocessing}
\label{sec:data-generation-preprocessing}
  Our synthetic dataset is generated following the same method used in previous studies~\cite{Zhang-Conversational-WWW20,wu-2021-clustering-of-conversational,Wang-2023-Efficient}.
  We set the dimension \(d=50\), the number of users (i.e., the number of unknown preference vectors) \(N=200\), the number of arms \(|\mathcal{A}|=5000\), and the number of key terms \(|\mathcal{K}|=1000\).
  To reflect the fact that each key term is related to a subset of arms, we generate them as follows.
  First we index each arm and key term by \(i \in [|\mathcal{A}|]\) and \(k \in [|\mathcal{K}|]\), respectively.
  We sample \(|\mathcal{K}|\) pseudo feature vectors \(\set{\dot{\vec{x}}_{k}}_{k \in [|\mathcal{K}|]}\), where each dimension of \(\dot{\vec{x}}_{k}\) is drawn from a uniform distribution \(\mathcal{U}(-1,1)\).
  Then for arm \(i\), we uniformly sample a subset of \(n_{i}\) key term indices \(\mathcal{K}_{i} \subset [|\mathcal{K}|]\), and assign the weight \(w_{i,k}=1/n_i\) for each \(k \in \mathcal{K}_{i}\), where \(n_{i}\) is a random integer selected from \(\set{1, 2, \dots, 5}\).
  Finally, the feature vector of arm \(i\), denoted by \(\vec{a}_i\), is drawn from \(\mathcal{N}(\sum_{j \in \mathcal{K}_{i}} \dot{\vec{x}}_{j}/n_{i}, \vec{I})\).
  And the feature vector of key term \(k\), denoted by \(\vec{x}_k\), is computed as \(\vec{x}_k=\sum_{i \in [|\mathcal{A}|]} \frac{w_{i,k}}{\sum_{j \in [|\mathcal{A}|]} w_{j,k}} \vec{a}_{i}\).
  Each user \(u \in [N]\) is represented as a preference vector \(\vec{\theta}_{u} \in \RR^d\).
  We generate \(N\) preference vectors by sampling each dimension of \(\vec{\theta}_{u}\) from \(\mathcal{U}(-1,1)\).

  For the real-world datasets, we regard movies\slash artists\slash businesses as arms.
  To facilitate data analysis, a subset of \(|\mathcal{A}|\) arms with the most quantity of user-assigned ratings/tags, and a subset of \(N\) users who assign the most quantity of ratings/tags, are extracted.
  Key terms are identified by using the associated movie genres, business categories, or tag IDs in the MovieLens, Yelp, and Last.fm datasets, respectively.
  For example, each movie is associated with a list of genres, such as ``action'', ``comedy'', or ``drama'', and each business (e.g., restaurant) has a list of categories, such as ``Mexican'', ``Burgers'', or ``Gastropubs''.
  Using the data extracted above, we create a \emph{feedback matrix} \(\vec{R}\) of size \({N \times |\mathcal{A}|}\), where each element \(\vec{R}_{i,j}\) represents the user \(i\)'s feedback to arm \(j\).
  We assume that the user's feedback is binary.
  For the MovieLens and Yelp datasets, a user's feedback for a movie/business is 1 if the user's rating is higher than 3; otherwise, the feedback is 0.
  For the Last.fm dataset, a user's feedback for an artist is 1 if the user assigns a tag to the artist.

  Next, we generate the feature vectors for arms \(\vec{a}_i\) and the preference vectors for users \(\vec{\theta}_u\).
  Following the existing works~\cite{wu-2021-clustering-of-conversational,Wang-2023-Efficient,zhao-2022-knowledge-aware}, we decompose the feedback matrix \(\vec{R}\) using Singular Value Decomposition (SVD) as \(\vec{R}=\vec{\Theta} \vec{S} \vec{A}^\mathsf{T}\), where \(\vec{\Theta} = \set{\vec{\theta}_u}_{u \in [N]}\), and \(\vec{A} = \set{\vec{a}_i}_{i \in [|\vec{A}|]}\).
  Then the top \(d=50\) dimensions of these vectors associated with the highest singular values in \(\vec{S}\) are extracted.
  The feature vectors for key terms are generated following \citet{Zhang-Conversational-WWW20}, which maintains equal weights for all key terms corresponding to each arm.

\section{Related Work}
\label{sec:related-work}
  \noindent\textbf{Linear Contextual Bandit.} Linear contextual bandits tackle the exploitation-exploration trade-off of online decision-making problems and have various applications such as recommender systems, clinical trials, and networking.
  This model has been extensively studied in both the infinite-arm setting~\cite{Abbasi-Improved-2011,li-2010-a-contextual} and the finite-arm setting~\cite{chu-2011-contextual}.
  Conversational contextual bandits, initially introduced by~\citet{Zhang-Conversational-WWW20}, allow recommenders to launch conversations and obtain feedback on key terms to speed up user preference elicitation, in addition to performing arm selection by the users.
  There have been many follow-up studies since then.
  \texttt{CtoF-ConUCB}~\cite{wu-2021-clustering-of-conversational} introduces clustering of conversational bandits to avoid the human labeling efforts and automatically learn the key terms with proper granularity.
  \texttt{RelativeConUCB}~\cite{xie-2021-comparison-based} proposes comparison-based conversational interactions to get comparative feedback from users.
  \texttt{GraphConUCB}~\cite{zhao-2022-knowledge-aware} integrates knowledge graphs as additional knowledge and uses D-optimal design to select key terms.
  \texttt{Hier-UCB} and \texttt{Hier-LinUCB}~\cite{zuo-2022-hierarchical-conversational} leverage the hierarchical structure between key terms and items to learn which items to recommend efficiently.
  \texttt{DecUCB}~\cite{xia-2023-user-regulation} uses causal inference to handle the bias resulting from items and key terms.
  \texttt{ConLinUCB}~\cite{Wang-2023-Efficient} proposes computing the \emph{barycentric spanner} of the key terms as an efficient exploration basis.
  Different from these existing works, our work extends conversational bandits into the federated setting, and we show that by leveraging key terms, both the regret and the computational/communication costs can be reduced.

  \noindent\textbf{Federated Linear Bandit.} Federated linear bandit has emerged as a popular research topic.
  It can be regarded as privacy-preserving distributed bandits without action collisions.
  Numerous works have been proposed to manage data aggregation and distribution~\cite{huang-2021-federated}, ensure differential privacy~\cite{dubey-2020-differentially,zhou-2023-differentially}, implement online clustering~\cite{liu-2022-federated}, and achieve asynchronous communication~\cite{li-2022-asynchronous,he-2022-simple,chen-2023-on-demand,yang-2024-federated}, etc.

  Our work is closely related to \citet{wang-2020-distributed-bandit} and \citet{huang-2021-federated}, since we all consider the setup where the agents communicate with a central server by sending and receiving packets, and each agent has a \emph{fixed} and \emph{finite} arm set.
  However, \citet{wang-2020-distributed-bandit} does not consider data heterogeneity, i.e., it assumes that all the clients have the same arm set.
  \citet{huang-2021-federated} mitigates this by allowing for different arm sets among clients.
  The authors propose \emph{multi-client G-optimal design} to deal with the challenges of data heterogeneity, as we mentioned in Section~\ref{sec:challenges}.
  But it requires all the clients to upload their active arm sets to the server, which increases communication costs and potentially compromises users' privacy.
  Moreover, as mentioned by the authors (Appendix G.2, Theorem 8 of \citet{huang-2021-federated}), for the \emph{shared-parameter} case (which is what we consider in this work), currently there is no known algorithm to efficiently compute the \emph{multi-client G-optimal design}.
  On the contrary, our work neither requires to upload each client's arm set, nor computes the \emph{multi-client G-optimal design}.
  Instead, we leverage the key terms to minimize uncertainty across each direction of the feature space, achieving the same regret upper bound and an improved communication cost upper bound.

\section{Supplementary Plots}
\subsection{Regret with Different Number of Clients}
\label{appendix:plot-clients}
  We provide the cumulative regret for $T=6,000$ when the number of clients varies from 3 to 15 using the three real-world datasets.
  As shown in Figure~\ref{fig:regret-vs-client-num-real-world-dataset}, the results are in line with that of the synthetic dataset and further demonstrate the performance of our algorithm \fedconpe.

  \begin{figure}[htb]
    \centering
    \begin{minipage}{0.48\columnwidth}
      \centering
      \subfigure[MovieLens]{
        \includegraphics[width=\columnwidth]{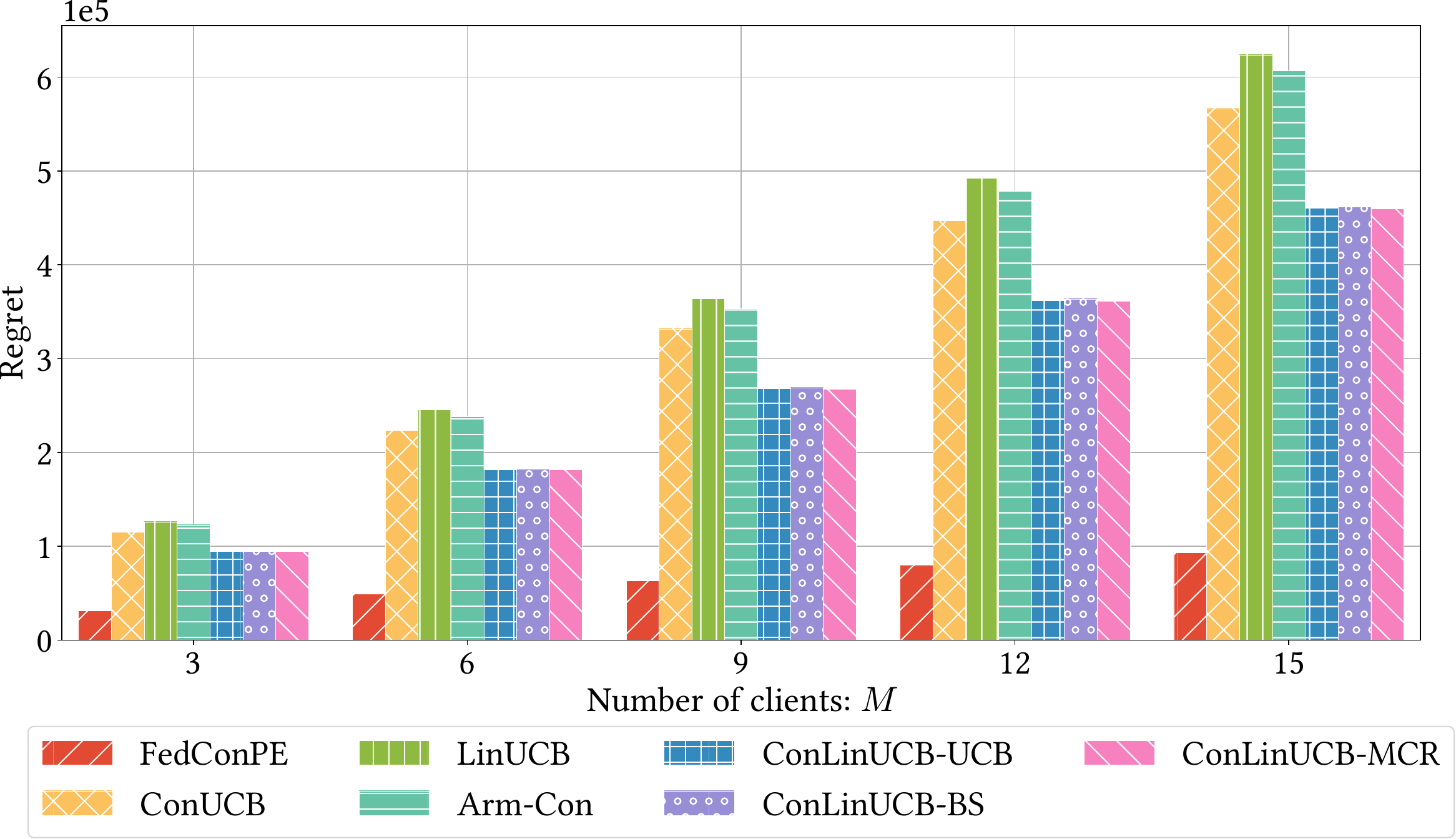}
      } \\
      \subfigure[Yelp]{
        \includegraphics[width=\columnwidth]{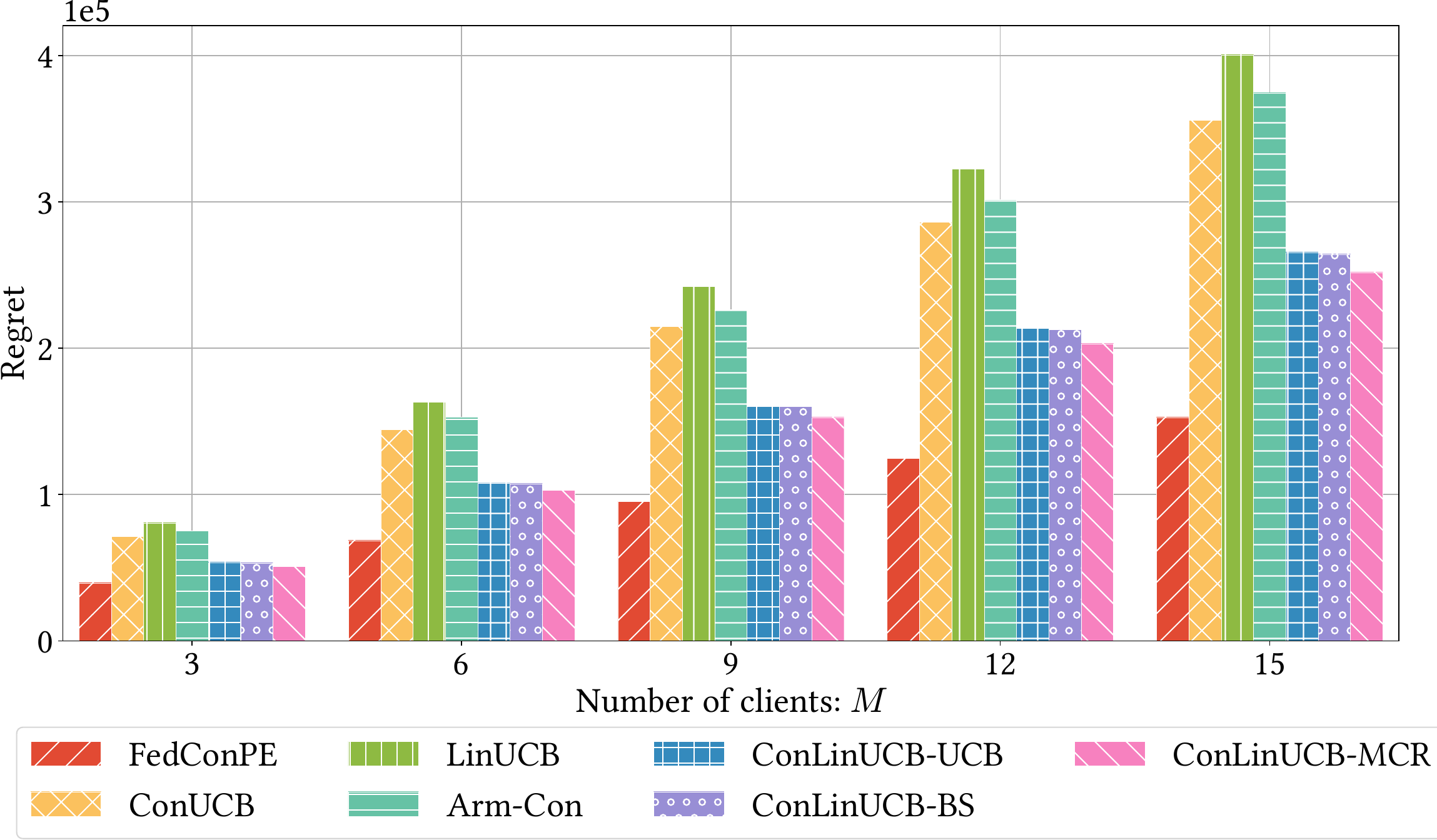}
      } \\
      \subfigure[Last.fm]{
        \includegraphics[width=\columnwidth]{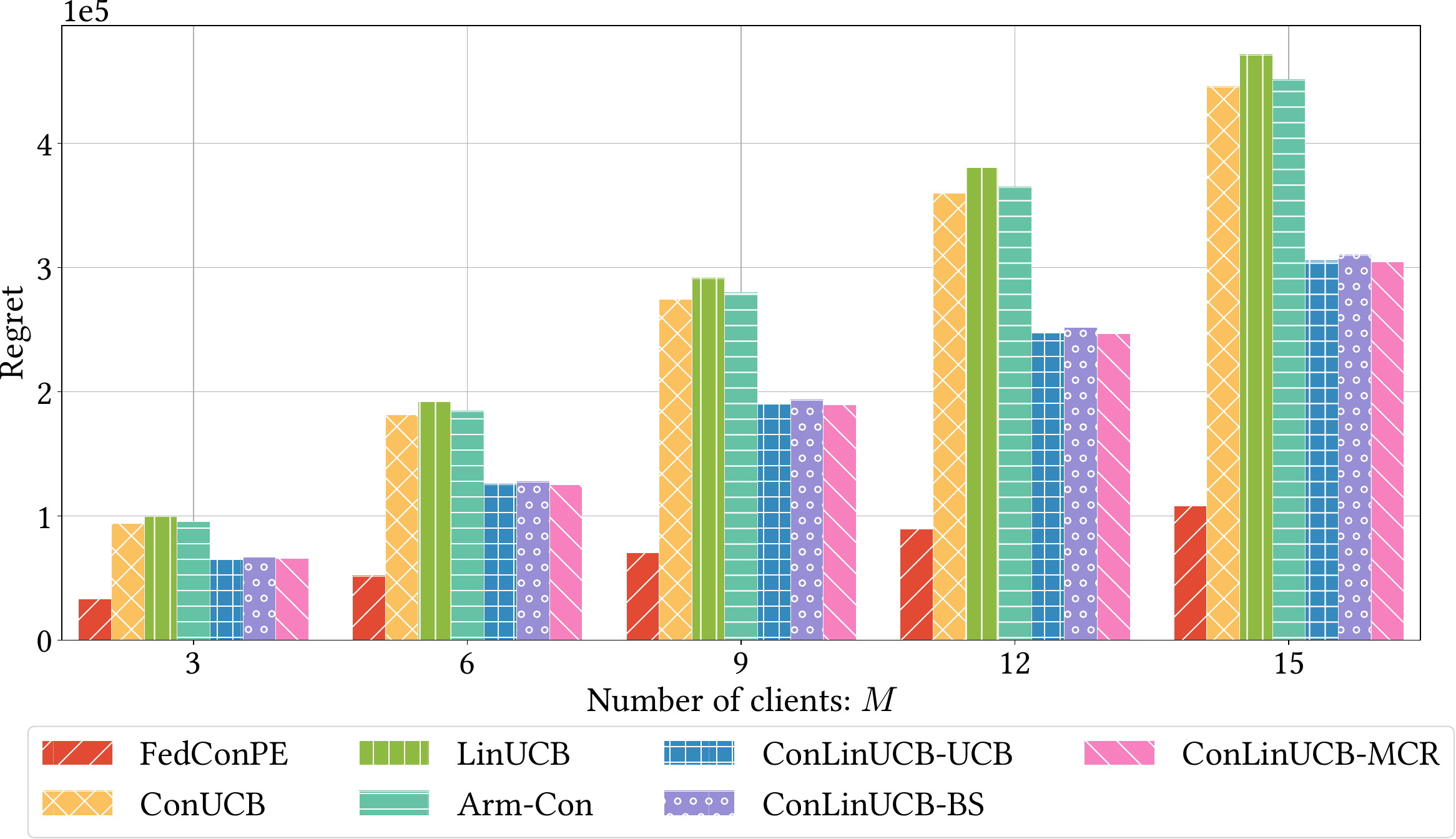}
      }
      \caption{\label{fig:regret-vs-client-num-real-world-dataset} Regret v.s. number of clients on real-world datasets.}
    \end{minipage}
    \hfill
    \begin{minipage}{0.48\columnwidth}
      \centering
      \subfigure[MovieLens]{
        \includegraphics[width=\columnwidth]{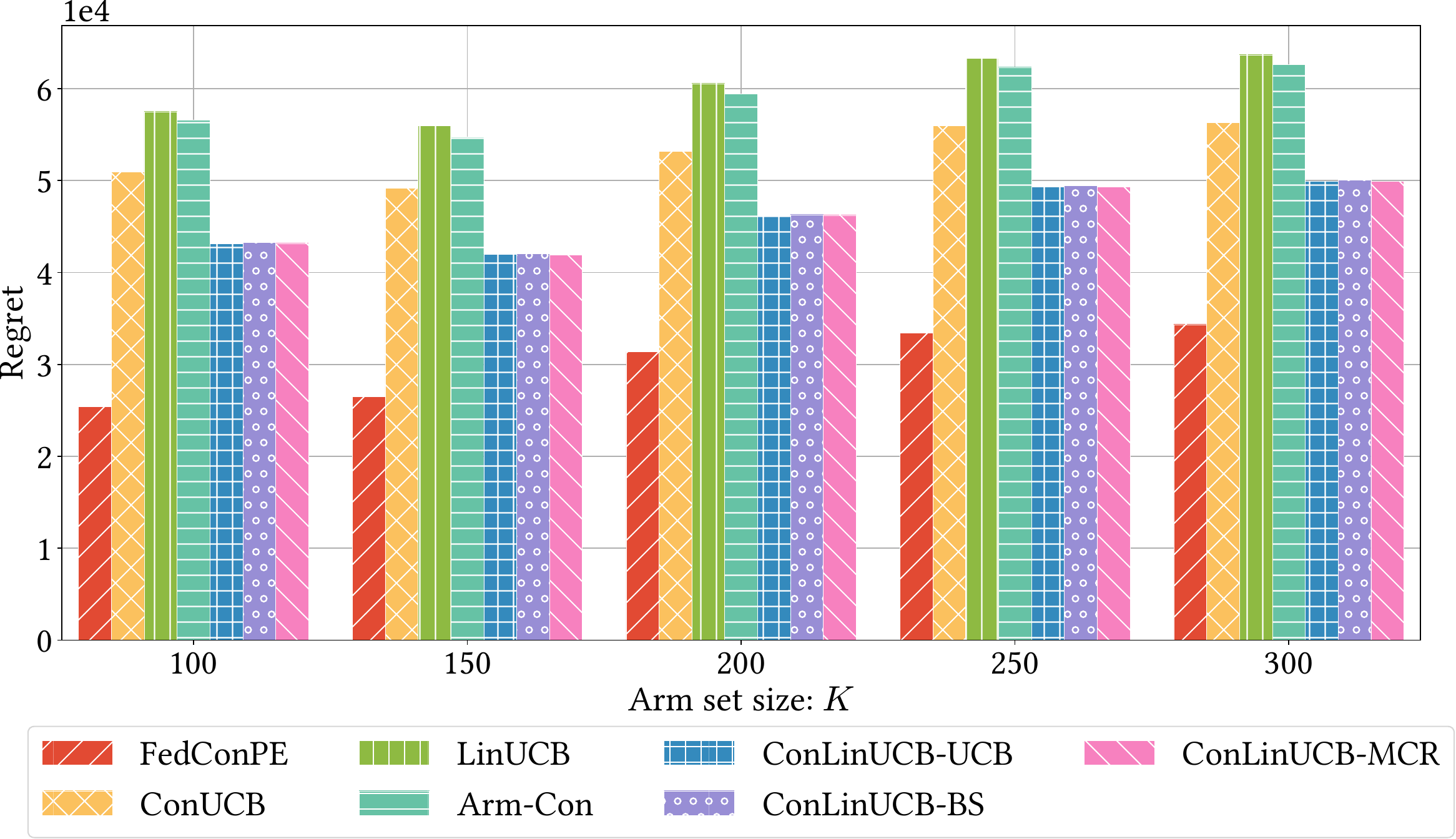}
      } \\
      \subfigure[Yelp]{
        \includegraphics[width=\columnwidth]{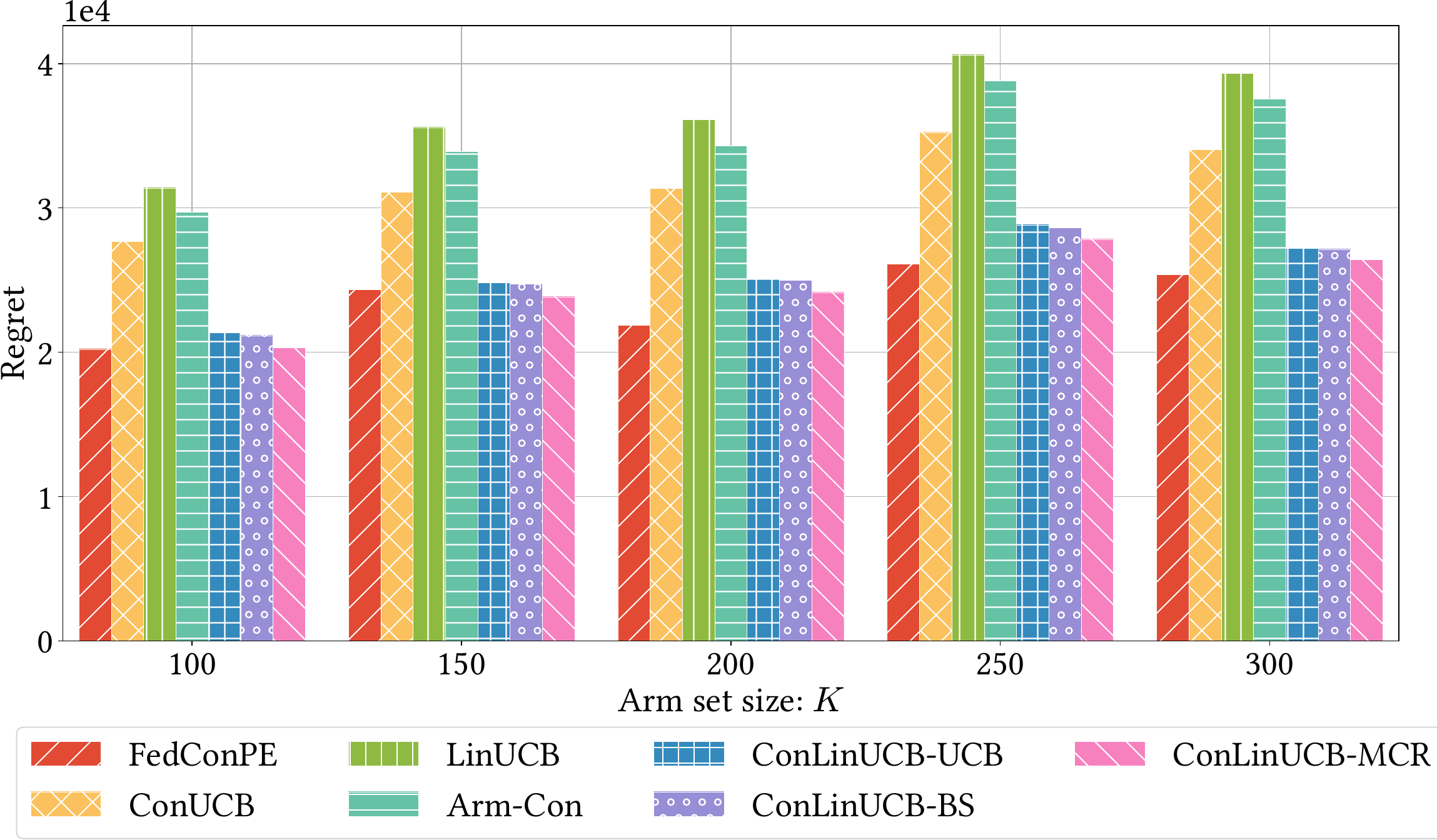}
      } \\
      \subfigure[Last.fm]{
        \includegraphics[width=\columnwidth]{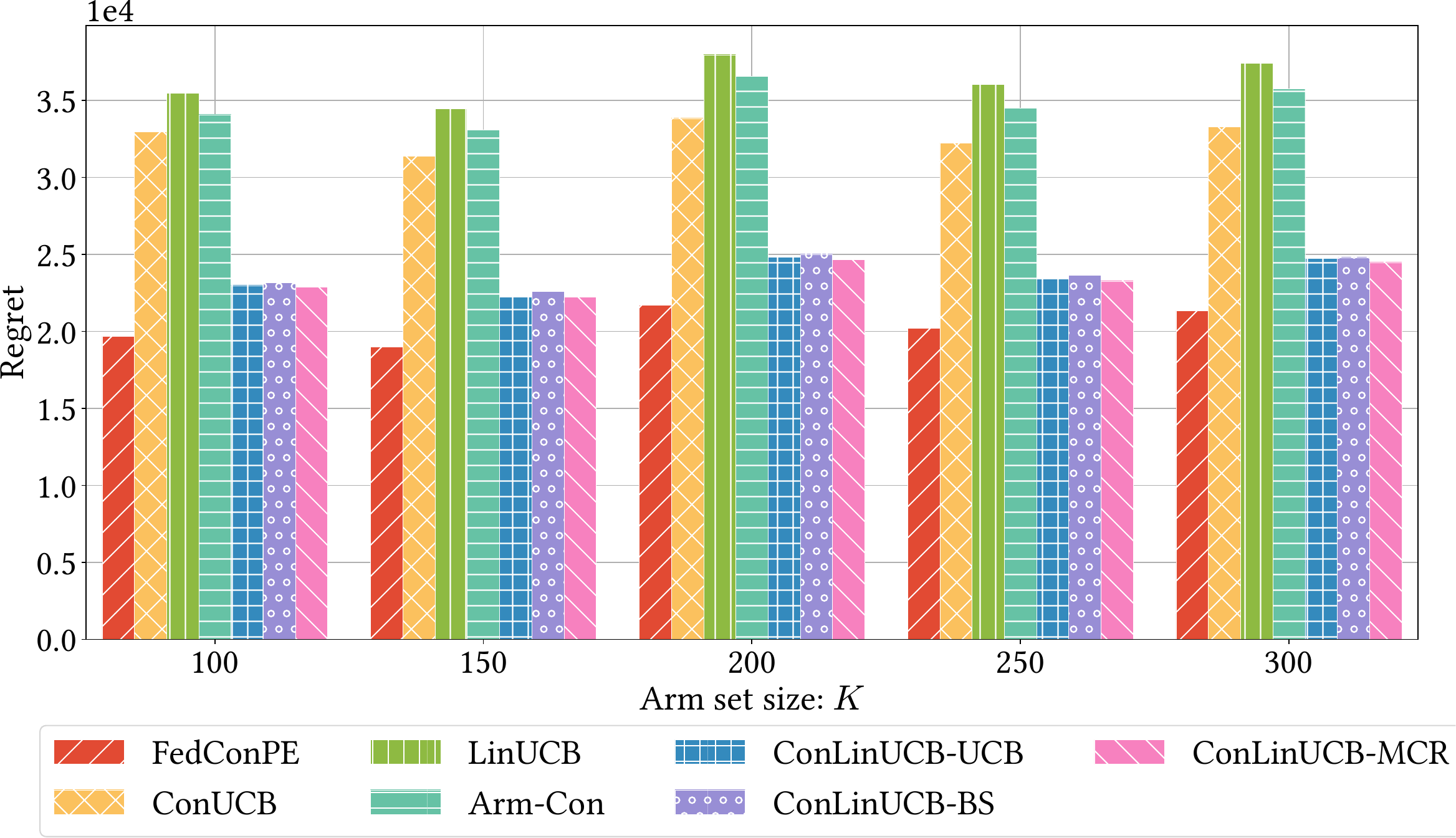}
      }
      \caption{\label{fig:regret-vs-poolsize-real-world-dataset} Regret v.s. size of arm sets on real-world datasets.}
    \end{minipage}
  \end{figure}

\subsection{Regret with Different Arm Set Sizes}
 \label{appendix:plot-arms}
  We provide the cumulative regret for $T=10,000$ when the arm set size varies from 100 to 300, using the three real-world datasets.
  As shown in Figure~\ref{fig:regret-vs-poolsize-real-world-dataset}, the results are in line with that of the synthetic dataset, showing that our algorithm \fedconpe is robust to the size of arm sets.

%% file: ijcai24.bbl
\begin{thebibliography}{}

\bibitem[\protect\citeauthoryear{Abbasi-Yadkori \bgroup \em et al.\egroup
  }{2011}]{Abbasi-Improved-2011}
Yasin Abbasi-Yadkori, D\'{a}vid P\'{a}l, and Csaba Szepesv\'{a}ri.
\newblock Improved algorithms for linear stochastic bandits.
\newblock In {\em Proceedings of the 24th International Conference on Neural
  Information Processing Systems}, NIPS'11, page 2312–2320, 2011.

\bibitem[\protect\citeauthoryear{Bretagnolle and
  Huber}{1978}]{bretagnolle-huber-1978}
J.~Bretagnolle and C.~Huber.
\newblock Estimation des densit{\'e}s : Risque minimax.
\newblock In C.~Dellacherie, P.~A. Meyer, and M.~Weil, editors, {\em
  S{\'e}minaire de Probabilit{\'e}s XII}, pages 342--363, Berlin, Heidelberg,
  1978. Springer Berlin Heidelberg.

\bibitem[\protect\citeauthoryear{Cantador \bgroup \em et al.\egroup
  }{2011}]{cantador-2011-second-workshop}
Ivan Cantador, Peter Brusilovsky, and Tsvi Kuflik.
\newblock Second workshop on information heterogeneity and fusion in
  recommender systems (hetrec2011).
\newblock In {\em Proceedings of the Fifth ACM Conference on Recommender
  Systems}, RecSys '11, page 387–388, 2011.

\bibitem[\protect\citeauthoryear{Chen \bgroup \em et al.\egroup
  }{2023}]{chen-2023-on-demand}
Yu-Zhen~Janice Chen, Lin Yang, Xuchuang Wang, Xutong Liu, Mohammad Hajiesmaili,
  John~CS Lui, and Don Towsley.
\newblock On-demand communication for asynchronous multi-agent bandits.
\newblock In {\em International Conference on Artificial Intelligence and
  Statistics}, pages 3903--3930. PMLR, 2023.

\bibitem[\protect\citeauthoryear{Christakopoulou \bgroup \em et al.\egroup
  }{2016}]{christakopoulou-2016-towards-conversational}
Konstantina Christakopoulou, Filip Radlinski, and Katja Hofmann.
\newblock Towards conversational recommender systems.
\newblock In {\em Proceedings of the 22nd ACM SIGKDD International Conference
  on Knowledge Discovery and Data Mining}, KDD '16, page 815–824, 2016.

\bibitem[\protect\citeauthoryear{Chu \bgroup \em et al.\egroup
  }{2011}]{chu-2011-contextual}
Wei Chu, Lihong Li, Lev Reyzin, and Robert Schapire.
\newblock Contextual bandits with linear payoff functions.
\newblock In {\em Proceedings of the Fourteenth International Conference on
  Artificial Intelligence and Statistics}, pages 208--214, 2011.

\bibitem[\protect\citeauthoryear{Dubey and
  Pentland}{2020}]{dubey-2020-differentially}
Abhimanyu Dubey and Alex Pentland.
\newblock Differentially-private federated linear bandits.
\newblock In {\em Proceedings of the 34th International Conference on Neural
  Information Processing Systems}, NIPS'20, 2020.

\bibitem[\protect\citeauthoryear{Frank and Wolfe}{1956}]{frank-wolfe-1956}
Marguerite Frank and Philip Wolfe.
\newblock An algorithm for quadratic programming.
\newblock {\em Naval Research Logistics Quarterly}, 3(1-2):95--110, 1956.

\bibitem[\protect\citeauthoryear{Gao \bgroup \em et al.\egroup
  }{2021}]{gao-2021-advances-challenges}
Chongming Gao, Wenqiang Lei, Xiangnan He, Maarten {de Rijke}, and Tat-Seng
  Chua.
\newblock Advances and challenges in conversational recommender systems: A
  survey.
\newblock {\em AI Open}, 2:100--126, 2021.

\bibitem[\protect\citeauthoryear{Harper and
  Konstan}{2015}]{harper-2015-movielens}
F.~Maxwell Harper and Joseph~A. Konstan.
\newblock The movielens datasets: History and context.
\newblock {\em ACM Trans. Interact. Intell. Syst.}, 5(4), dec 2015.

\bibitem[\protect\citeauthoryear{He \bgroup \em et al.\egroup
  }{2022}]{he-2022-simple}
Jiafan He, Tianhao Wang, Yifei Min, and Quanquan Gu.
\newblock A simple and provably efficient algorithm for asynchronous federated
  contextual linear bandits.
\newblock In {\em Advances in Neural Information Processing Systems},
  volume~35, pages 4762--4775, 2022.

\bibitem[\protect\citeauthoryear{Huang \bgroup \em et al.\egroup
  }{2021}]{huang-2021-federated}
Ruiquan Huang, Weiqiang Wu, Jing Yang, and Cong Shen.
\newblock Federated linear contextual bandits.
\newblock {\em Advances in neural information processing systems},
  34:27057--27068, 2021.

\bibitem[\protect\citeauthoryear{Kiefer and
  Wolfowitz}{1960}]{kiefer-wolfowitz-1960}
J.~Kiefer and J.~Wolfowitz.
\newblock The equivalence of two extremum problems.
\newblock {\em Canadian Journal of Mathematics}, 12:363–366, 1960.

\bibitem[\protect\citeauthoryear{Lattimore and
  Szepesvári}{2020}]{lattimore-2020-bandit-algorithms}
Tor Lattimore and Csaba Szepesvári.
\newblock {\em Bandit Algorithms}.
\newblock Cambridge University Press, 2020.

\bibitem[\protect\citeauthoryear{Lei \bgroup \em et al.\egroup
  }{2020}]{lei-2020-conversational-recommendation}
Wenqiang Lei, Xiangnan He, Maarten de~Rijke, and Tat-Seng Chua.
\newblock Conversational recommendation: Formulation, methods, and evaluation.
\newblock In {\em Proceedings of the 43rd International ACM SIGIR Conference on
  Research and Development in Information Retrieval}, SIGIR '20, page
  2425–2428, 2020.

\bibitem[\protect\citeauthoryear{Li and Wang}{2022}]{li-2022-asynchronous}
Chuanhao Li and Hongning Wang.
\newblock Asynchronous upper confidence bound algorithms for federated linear
  bandits.
\newblock In {\em International Conference on Artificial Intelligence and
  Statistics}, pages 6529--6553, 2022.

\bibitem[\protect\citeauthoryear{Li \bgroup \em et al.\egroup
  }{2010}]{li-2010-a-contextual}
Lihong Li, Wei Chu, John Langford, and Robert~E. Schapire.
\newblock A contextual-bandit approach to personalized news article
  recommendation.
\newblock In {\em Proceedings of the 19th International Conference on World
  Wide Web}, WWW '10, page 661–670, 2010.

\bibitem[\protect\citeauthoryear{Li \bgroup \em et al.\egroup
  }{2024}]{li-2024-fedconpe}
Zhuohua Li, Maoli Liu, and John C.~S. Lui.
\newblock Fedconpe: Efficient federated conversational bandits with
  heterogeneous clients, 2024.

\bibitem[\protect\citeauthoryear{Lin and Moothedath}{2023}]{lin-2023-federated}
Jiabin Lin and Shana Moothedath.
\newblock Federated stochastic bandit learning with unobserved context, 2023.

\bibitem[\protect\citeauthoryear{Liu \bgroup \em et al.\egroup
  }{2022}]{liu-2022-federated}
Xutong Liu, Haoru Zhao, Tong Yu, Shuai Li, and John~CS Lui.
\newblock Federated online clustering of bandits.
\newblock In {\em Uncertainty in Artificial Intelligence}, pages 1221--1231,
  2022.

\bibitem[\protect\citeauthoryear{Sun and
  Zhang}{2018}]{sun-2018-conversational-recommender}
Yueming Sun and Yi~Zhang.
\newblock Conversational recommender system.
\newblock In {\em The 41st International ACM SIGIR Conference on Research \&
  Development in Information Retrieval}, SIGIR '18, page 235–244, 2018.

\bibitem[\protect\citeauthoryear{Todd}{2016}]{todd-2016-minimum-volume}
Michael~J. Todd.
\newblock {\em Minimum-Volume Ellipsoids}.
\newblock Society for Industrial and Applied Mathematics, 2016.

\bibitem[\protect\citeauthoryear{Wang \bgroup \em et al.\egroup
  }{2020}]{wang-2020-distributed-bandit}
Yuanhao Wang, Jiachen Hu, Xiaoyu Chen, and Liwei Wang.
\newblock Distributed bandit learning: Near-optimal regret with efficient
  communication.
\newblock In {\em 8th International Conference on Learning Representations},
  ICLR'20, 2020.

\bibitem[\protect\citeauthoryear{Wang \bgroup \em et al.\egroup
  }{2023}]{Wang-2023-Efficient}
Zhiyong Wang, Xutong Liu, Shuai Li, and John C.~S. Lui.
\newblock Efficient explorative key-term selection strategies for
  conversational contextual bandits.
\newblock {\em Proceedings of the AAAI Conference on Artificial Intelligence},
  37(8):10288--10295, Jun. 2023.

\bibitem[\protect\citeauthoryear{Weyl}{1912}]{weyl-1912}
Hermann Weyl.
\newblock Das asymptotische verteilungsgesetz der eigenwerte linearer
  partieller differentialgleichungen (mit einer anwendung auf die theorie der
  hohlraumstrahlung).
\newblock {\em Mathematische Annalen}, 71(4):441 – 479, 1912.

\bibitem[\protect\citeauthoryear{Wu \bgroup \em et al.\egroup
  }{2021}]{wu-2021-clustering-of-conversational}
Junda Wu, Canzhe Zhao, Tong Yu, Jingyang Li, and Shuai Li.
\newblock Clustering of conversational bandits for user preference learning and
  elicitation.
\newblock In {\em Proceedings of the 30th ACM International Conference on
  Information \& Knowledge Management}, CIKM '21, page 2129–2139, 2021.

\bibitem[\protect\citeauthoryear{Xia \bgroup \em et al.\egroup
  }{2023}]{xia-2023-user-regulation}
Yu~Xia, Junda Wu, Tong Yu, Sungchul Kim, Ryan~A. Rossi, and Shuai Li.
\newblock User-regulation deconfounded conversational recommender system with
  bandit feedback.
\newblock In {\em Proceedings of the 29th ACM SIGKDD Conference on Knowledge
  Discovery and Data Mining}, KDD '23, page 2694–2704, 2023.

\bibitem[\protect\citeauthoryear{Xie \bgroup \em et al.\egroup
  }{2021}]{xie-2021-comparison-based}
Zhihui Xie, Tong Yu, Canzhe Zhao, and Shuai Li.
\newblock Comparison-based conversational recommender system with relative
  bandit feedback.
\newblock In {\em Proceedings of the 44th International ACM SIGIR Conference on
  Research and Development in Information Retrieval}, SIGIR '21, page
  1400–1409, 2021.

\bibitem[\protect\citeauthoryear{Yang \bgroup \em et al.\egroup
  }{2024}]{yang-2024-federated}
Hantao Yang, Xutong Liu, Zhiyong Wang, Hong Xie, John C.~S. Lui, Defu Lian, and
  Enhong Chen.
\newblock Federated contextual cascading bandits with asynchronous
  communication and heterogeneous users.
\newblock {\em Proceedings of the AAAI Conference on Artificial Intelligence},
  38(18):20596--20603, Mar. 2024.

\bibitem[\protect\citeauthoryear{Zhang \bgroup \em et al.\egroup
  }{2020}]{Zhang-Conversational-WWW20}
Xiaoying Zhang, Hong Xie, Hang Li, and John C.S.~Lui.
\newblock Conversational contextual bandit: Algorithm and application.
\newblock In {\em Proceedings of The Web Conference 2020}, WWW '20, page
  662–672, 2020.

\bibitem[\protect\citeauthoryear{Zhao \bgroup \em et al.\egroup
  }{2022}]{zhao-2022-knowledge-aware}
Canzhe Zhao, Tong Yu, Zhihui Xie, and Shuai Li.
\newblock Knowledge-aware conversational preference elicitation with bandit
  feedback.
\newblock In {\em Proceedings of the ACM Web Conference 2022}, page 483–492,
  2022.

\bibitem[\protect\citeauthoryear{Zhou and
  Chowdhury}{2023}]{zhou-2023-differentially}
Xingyu Zhou and Sayak~Ray Chowdhury.
\newblock On differentially private federated linear contextual bandits, 2023.

\bibitem[\protect\citeauthoryear{Zuo \bgroup \em et al.\egroup
  }{2022}]{zuo-2022-hierarchical-conversational}
Jinhang Zuo, Songwen Hu, Tong Yu, Shuai Li, Handong Zhao, and Carlee Joe-Wong.
\newblock Hierarchical conversational preference elicitation with bandit
  feedback.
\newblock In {\em Proceedings of the 31st ACM International Conference on
  Information \& Knowledge Management}, CIKM '22, page 2827–2836, 2022.

\end{thebibliography}
